\documentclass[10pt,twocolumn,letterpaper]{article}

 \usepackage{cvpr}              

\usepackage{graphicx}

\usepackage{amssymb}
\usepackage{booktabs}
\usepackage{tabularx}
\usepackage{wrapfig}
\usepackage{threeparttable}
\usepackage{multirow}
\usepackage{colortbl}
\usepackage{adjustbox}
\usepackage{tablefootnote}
\usepackage{float}
\usepackage{placeins}
\usepackage{lipsum}
\usepackage{bm}
\usepackage{amsmath}
\usepackage{pgfplots}
\usepackage{tikz}
\usepackage{dirtytalk}

\definecolor{Gray}{gray}{0.9}

\definecolor{cvprblue}{rgb}{0.21,0.49,0.74}
\usepackage[pagebackref,breaklinks,colorlinks,allcolors=cvprblue]{hyperref}

\usepackage[capitalize]{cleveref}
\crefname{section}{Sec.}{Secs.}
\Crefname{section}{Section}{Sections}
\Crefname{table}{Table}{Tables}
\crefname{table}{Tab.}{Tabs.}

\usepackage{amsmath, amssymb, bm}
\pgfplotsset{compat=1.18} 
\definecolor{Gray}{gray}{0.9}

\newcounter {chapter}
\newcounter{theorem}
\renewcommand\thechapter{\@arabic\c@chapter}
\renewcommand\thetheorem{\thesection.\arabic{theorem}}

\newenvironment{theorem}[1][\theoremname]%
    {\refstepcounter{theorem}\par\medskip\noindent%
    \textbf{#1~\thetheorem}\:\itshape}{\medskip}

\newenvironment{proof}[1][\proofname]%
    {\par\medskip\noindent%
    \textit{#1}\:\ignorespaces}{\hfill$\square$\medskip}

\newtheorem{lemma}[theorem]{Lemma}
\makeatletter
\def\theopargself{
    \def\@spopargbegintheorem##1##2##3##4##5{
        \trivlist\item[\hskip \labelsep {##4 ##1\ ##2}]
        {\ignorespaces ##4 ##3 \@thmcounterend\ }##5 \ignorespaces}
    
    \def\@Opargbegintheorem##1##2##3##4{
        \trivlist \item[\hskip \labelsep {##3 ##1}]
        {\ignorespaces ##3(##2) \@thmcounterend\ }##4 \ignorespaces}
}

\def\@spthm#1#2#3#4{\topsep 7\p@ \@plus2\p@ \@minus4\p@
\refstepcounter{#1}%
\@ifnextchar[{\@spythm{#1}{#2}{#3}{#4}}{\@spxthm{#1}{#2}{#3}{#4}}}

\def\@spxthm#1#2#3#4{\@spbegintheorem{#2}{\csname the#1\endcsname}{#3}{#4}%
                    \ignorespaces}

\def\@spythm#1#2#3#4[#5]{\@spopargbegintheorem{#2}{\csname
       the#1\endcsname}{#5}{#3}{#4}\ignorespaces}

\def\@spbegintheorem#1#2#3#4{\trivlist
                 \item[\hskip\labelsep{#3#1\ #2\@thmcounterend}]#4}

\def\@spopargbegintheorem#1#2#3#4#5{\trivlist
      \item[\hskip\labelsep{#4#1\ #2}]{#4(#3)\@thmcounterend\ }#5}
    
\makeatother

\newcommand\theoremname{Theorem}
\newcommand\proofname{Proof}

\def\paperID{*****} 
\def\confName{CVPR}
\def\confYear{2025}

\title{A Tale of Two Classes: \\ Adapting Supervised Contrastive Learning to Binary Imbalanced Datasets}

\author{
  David Mildenberger\textsuperscript{1,2,*}, Paul Hager\textsuperscript{1,*},
  Daniel Rueckert\textsuperscript{1,2,3}, Martin J. Menten\textsuperscript{1,2,3}\\[0.5em]
  \normalsize \textsuperscript{1}Technical University of Munich,
  \textsuperscript{2}Munich Center for Machine Learning, \textsuperscript{3}Imperial College London\\[0.5em]
  \small \texttt{\{david.mildenberger, paul.hager, daniel.rueckert, martin.menten\}@tum.de}\\[0.5em]
  \small \textsuperscript{*}These authors contributed equally.
}

\begin{document}
\maketitle
\begin{abstract}
Supervised contrastive learning (SupCon) has proven to be a powerful alternative to the standard cross-entropy loss for classification of multi-class balanced datasets. However, it struggles to learn well-conditioned representations of datasets with long-tailed class distributions. This problem is potentially exacerbated for binary imbalanced distributions, which are commonly encountered during many real-world problems such as medical diagnosis. In experiments on seven binary datasets of natural and medical images, we show that the performance of SupCon decreases with increasing class imbalance. To substantiate these findings, we introduce two novel metrics that evaluate the quality of the learned representation space. By measuring the class distribution in local neighborhoods, we are able to uncover structural deficiencies of the representation space that classical metrics cannot detect. Informed by these insights, we propose two new supervised contrastive learning strategies tailored to binary imbalanced datasets that improve the structure of the representation space and increase downstream classification accuracy over standard SupCon by up to 35\%. We make our code available.\footnote{\url{https://github.com/aiforvision/TTC}}
\end{abstract}
    
\section{Introduction}

\begin{figure}[tb]
  \centering
  \includegraphics[width=1.0\linewidth]{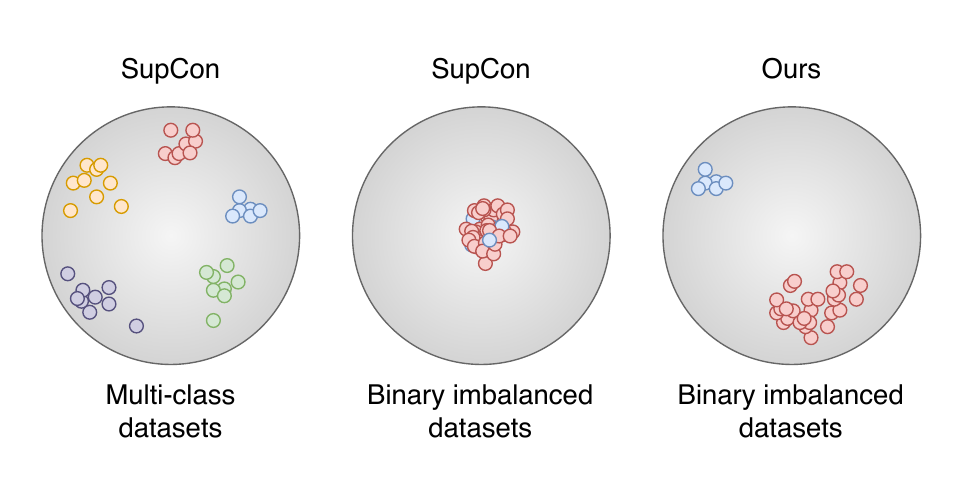}
  \caption{Supervised contrastive learning (SupCon) on multi-class balanced datasets returns a well-conditioned representation space, in which semantic classes are clearly separated. We show that for binary imbalanced datasets the prevalence of a dominant majority class causes the embeddings to collapse to a single point. Our proposed fixes restore the clear separation of semantic classes.}
  \label{fig:graphical_abstract}
\end{figure}


Supervised contrastive learning (SupCon) has emerged as a powerful alternative to the cross-entropy loss for supervised deep learning~\cite{DBLP:journals/corr/abs-2004-11362,graf2023dissecting,li2022targeted,li2022selective}. SupCon combines full label information with a contrastive loss to cluster samples of the same class in similar regions of the representation space. Conversely, embeddings of different classes are pushed apart. This results in a well-conditioned representation space that preserves discriminative features of each sample while separating semantic classes~\cite{assran2022hiddenuniformclusterprior}. 
SupCon has been used to achieve state-of-the-art results across diverse fields, including genetics \cite{camargo2024identification}, out-of-distribution detection \cite{sun2022out}, object detection \cite{sun2021fsce}, video action recognition \cite{han2020self}, and neuroscience \cite{schneider2023learnable}.

SupCon has been predominantly developed on and applied to multi-class balanced benchmark datasets like ImageNet, which consist of numerous equally prevalent classes. In contrast, real-world datasets often significantly deviate from these idealized conditions.
There has been a growing focus on adapting supervised contrastive learning to long-tailed datasets, which are characterized by a few common ``head'' classes and many rare ``tail'' classes~\cite{DBLP:journals/corr/abs-2107-12028, wang2021contrastive,li2022targeted,zhu2022balanced, hou2023subclass}. 
These works show that SupCon often yields representation spaces with dominating head classes when applied to long-tailed datasets, leading to reduced downstream utility. 

SupCon's shortcomings on long-tailed datasets are potentially exacerbated on distributions with only two underlying classes: a common majority class and a rare minority class. Such binary imbalanced distributions are common in real-world tasks, such as anomaly detection, fraud detection, and disease classification. For example, during medical screening, subjects are classified as either healthy or diseased, with the healthy cases typically outnumbering the pathological ones~\cite{healthcare10071293,Hee_2022}. A survey of representation learning for medical image classification found that 78 out of 114 studies focused on binary classification problems~\cite{Hee_2022}.

This work is the first to investigate the effectiveness of SupCon on binary imbalanced datasets, identifying limitations of existing supervised contrastive learning strategies and proposing algorithmic solutions to these issues. Our main contributions are:

\begin{itemize} 
\item We empirically demonstrate that SupCon is ineffective on binary imbalanced datasets. In controlled experiments on seven natural and medical imaging datasets, we observe that the performance of SupCon degrades with increasing class imbalance, falling behind the standard cross-entropy loss even at moderate levels of class imbalance.
\item To investigate these findings, we introduce two novel metrics for diagnosing the structural deficiencies of the representation space. We show that at high class imbalance all embeddings are closely clustered in a small region of the representation space, preventing separation of semantic classes (see \cref{fig:graphical_abstract}). While canonical metrics fail to capture this problem, our proposed metrics detect this representation space collapse and diminished downstream utility. Furthermore, we theoretically substantiate our empirical observations in a proof.
\item Informed by the insights gained through our metrics, we propose two new supervised contrastive learning strategies tailored to binary imbalanced datasets. These adjustments are easy to implement and incur minimal additional computational cost compared to the standard SupCon loss. We demonstrate that our fixes boost downstream classification performance by up to 35\% over SupCon and outperform leading strategies for long-tailed data by up to 5\%.
\end{itemize}

\section{Related works}

\subsection{Analyzing representation spaces}

To evaluate the quality of representation spaces, Wang and Isola have introduced the notions of representation space alignment and uniformity \cite{DBLP:journals/corr/abs-2005-10242}. Alignment measures how closely semantically similar samples are located in the representation space. Uniformity quantifies the utilization of the representation space's capacity. Both metrics have been empirically validated as strong indicators of the representation space's quality and downstream utility. However, high uniformity and alignment alone do not guarantee separability of classes, as shown by Wang \etal \cite{wang2022chaos}. In a subsequent study, Li \etal proposed analyzing alignment and uniformity at the level of semantic classes instead of individual samples \cite{li2022targeted}. While this improves upon sample-wise analysis, it still fails to properly compare representations \textit{between} classes, which is crucial for downstream classification performance. We address this limitation by introducing two novel metrics, enabling the evaluation of sample and class consistency within representation neighborhoods.

\subsection{Supervised contrastive learning for long-tailed datasets}
\label{sec:related_works:longtail}

When applying SupCon to long-tailed datasets, samples from the majority class often occupy a disproportionate amount of the representation space \cite{zhu2022balanced, li2022targeted, DBLP:journals/corr/abs-2107-12028,kang2021exploring, hou2023subclass}. In the most severe cases the representation space collapses completely, losing all utility \cite{pmlr-v202-xue23d,graf2023dissecting,Hager_2023_CVPR,chen2022perfectly}. Many works thus aim to enhance latent space uniformity by spreading features evenly, irrespective of data imbalance. Zhu \etal balance gradient contributions of classes to achieve a regular simplex latent structure \cite{zhu2022balanced}. Hou \etal split the majority classes into smaller sub-classes according to their latent features \cite{hou2023subclass}. Cui \etal leverage parametric class prototypes to adaptively balance the learning signal \cite{DBLP:journals/corr/abs-2107-12028}. Kang \etal and Li \etal limit the number of positives that contribute to the loss with Li \etal also using fixed class prototypes \cite{kang2021exploring, li2022targeted}. Although these methods outperform SupCon on long-tailed distributions, they remain untested on binary imbalanced distributions whose unique characteristics are not explicitly addressed.

\section{Metrics to diagnose representation spaces of binary data distributions}
\label{sec:unif_align}

\begin{figure*}[tb]
\centering
    \includegraphics[width=1\linewidth]{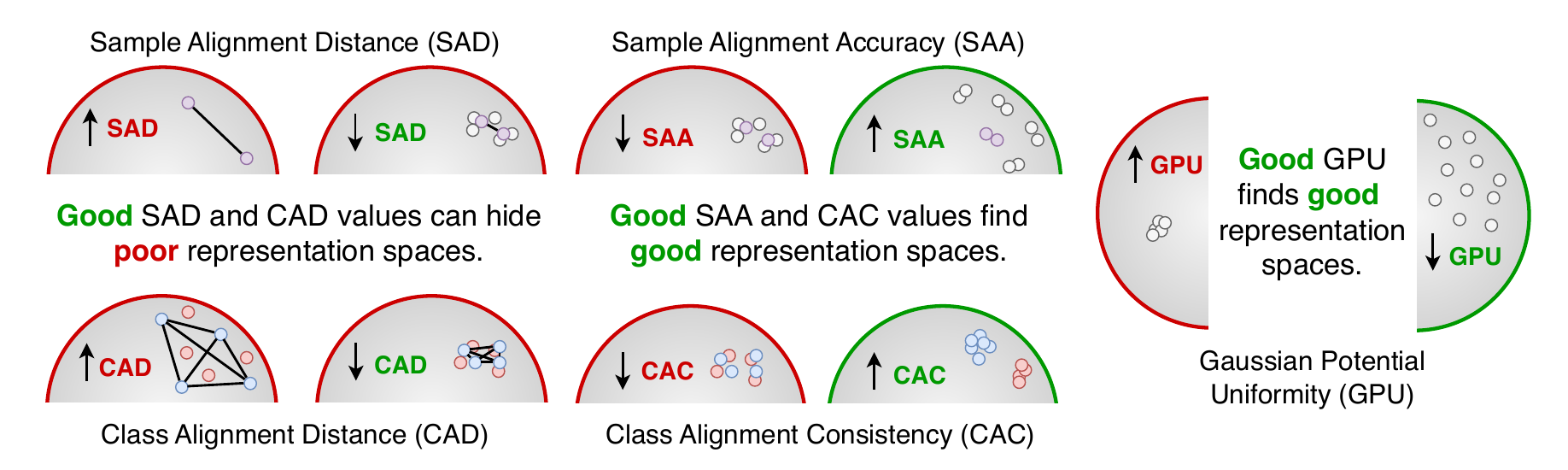}
    \caption{Our novel sample alignment accuracy (SAA) and class alignment consistency (CAC) metrics capture the relationships \textit{between} embeddings of different classes instead of just within one class. By more directly measuring the separability of latent classes, it is a stronger indicator of a representation space's downstream utility.}
    \label{fig:metrics}
\end{figure*}

Alignment and uniformity are established metrics for evaluating the quality and structure of representation spaces obtained via unsupervised contrastive learning \cite{DBLP:journals/corr/abs-2005-10242}. Alignment measures the average distance between positive pairs in the feature space. Low average distance, or high alignment, indicates consistent embeddings that are robust to noise. Uniformity measures the average Gaussian potential between embeddings on the unit hypersphere. Low average potential, or high uniformity, corresponds to expressive embeddings that fully utilize the entire representation space. Effective representation spaces exhibit both high alignment and high uniformity and should in theory yield good linear separability of the semantic classes.

However, because these metrics were originally developed for unsupervised contrastive learning, they operate on a per-sample level and ignore latent class information. To address this limitation, Li \etal have extended alignment and uniformity to SupCon, by analyzing the alignment of classes instead of samples \cite{li2022targeted}. Although this class-level alignment better reflects semantic separability, it still does not capture the relationships \emph{between} classes.

We therefore propose two new alignment metrics: sample alignment accuracy (SAA) and class alignment consistency (CAC). 
Our metrics compare the alignment both within and between samples and classes.
This is in contrast to the canonical sample alignment distance (SAD)~\cite{DBLP:journals/corr/abs-2005-10242} and class alignment distance (CAD)~\cite{li2022targeted} which only measure the alignment within one sample or class, as shown in \cref{fig:metrics}.

\subsection{Definitions}
\label{sec:definitions}

We define a dataset \(\mathcal{X}\) containing \(N\) images, \(\mathcal{X} = \{x_k\}_{k=1}^{N}\). Given that the images in \(\mathcal{X}\) are labeled with a binary class distribution, we denote the samples of class \(i \in \{0,1\}\) as \(x \in \mathcal{X}_i\) with the function $g: \mathcal{X} \rightarrow \{0,1\}$ mapping each image to its label. Let \(\Tilde{x}_k\) and \(\Tilde{x}_k^+\) denote two randomly augmented instances of \(x_k\) that form a positive pair. Conversely, any view that is not generated from \(x_k\), denoted by \(\Tilde{x}_k^-\) denotes, forms a negative pair with \(\Tilde{x}_k\). The set of all views is called $\mathcal{W}$ with $|\mathcal{W}| = 2N$. The set of all views of class $i$ is $\mathcal{W}_i$. The function \(f: \mathcal{W} \rightarrow \mathcal{S}^{d-1}\) maps a view \(\Tilde{x}\) onto a \(d\)-dimensional representation on the unit sphere \(\mathcal{S}^{d-1}\). 
%

\subsubsection{Sample alignment distance (SAD)} \label{def:sad} 
Sample alignment distance (SAD), as defined by Wang \etal~\cite{DBLP:journals/corr/abs-2005-10242}, measures the average distance between representations of two augmented views of the same sample. Formally, SAD is computed as the average pairwise \(\ell_2\) distance between sample-wise positive pairs:
\begin{equation}
   \text{SAD} = \frac{1}{|\mathcal{X}|} \sum_{x\in \mathcal{X}} \|(f(\Tilde{x}) - f(\Tilde{x}^+))\|_2
\end{equation}

A low SAD, or high alignment, implies that two different views of the same image produce similar embeddings. Obtaining consistent embeddings despite perturbations from augmentations typically indicates that generalized class-level semantic features are being captured, which is ultimately beneficial for downstream applications.


\subsubsection{Sample alignment accuracy (SAA)} We introduce the concept of sample alignment accuracy (SAA) to determine if the embeddings of positive pairs are more closely aligned with each other compared to other samples. 
$\text{SAA}$ is the proportion of all sample-wise positive pairs for which the \(\ell_2\) distance between their embeddings is smaller than that to all negative pairs:
\begin{multline}
\text{SAA} = \frac{1}{|\mathcal{X}|} \sum_{x \in \mathcal{X}} \mathbb{1}\Bigg(
    \|(f(\Tilde{x}) - f(\Tilde{x}^+))\|_2 \\
    < \min_{\Tilde{x}^{-} \in \mathcal{W} \setminus \{\Tilde{x},\Tilde{x}^+\}} 
    \|(f(\Tilde{x}) - f(\Tilde{x}^-))\|_2
\Bigg)
\end{multline}

\noindent Here, $\mathbb{1}(\cdot)$ is the indicator function, which outputs 1 if the condition $\cdot$ holds, and 0 otherwise.

Compared to SAD, SAA is more insightful in cases in which many samples, both positives and negatives, are placed in close proximity to each other. While SAD would indicate high alignment despite low separability of semantic classes, SAA would correctly diagnose a partially degenerate representation space.

\subsubsection{Class alignment distance (CAD)} Class alignment distance (CAD), introduced by Li \etal~\cite{li2022targeted}, calculates the average distance between all representations within a class to evaluate how well the learned representation space clusters samples according to their semantic labels \cite{li2022targeted}. Let $C$ be the number of classes and \(\Tilde{x}, \Tilde{x}' \in \mathcal{W}_i\) all unique pairs of samples in \(\mathcal{W}_i\):

\begin{equation}
   \text{CAD} = \frac{1}{C} \sum^C_{i=1} \frac{1}{\binom{|\mathcal{X}_i|+1}{2}} \sum_{\Tilde{x}, \Tilde{x}' \in \mathcal{W}_i}\|(f(\Tilde{x}) - f(\Tilde{x}'))\|_2
\end{equation}

Compared to SAD, CAD captures alignment across an entire class, indicating how well a class clusters on the representation space's hypersphere.
    
\subsubsection{Class alignment consistency (CAC)} 
To measure how pure embedding neighborhoods are with respect to the latent class, we introduce class alignment consistency (CAC). We define class alignment within a local neighborhood of the closest $r$ views to $\Tilde{x}$, which we call $\mathcal{W}_{\Tilde{x}}$. For our analysis we set $r$ to 5\% of all views. Let $\mathcal{D}$ denote the set of sets containing each $\Tilde{x}$ and its local neighborhood $\mathcal{W}_{\Tilde{x}}$, with $(\Tilde{x}, \mathcal{W}_{\Tilde{x}}) \in \mathcal{D}$. 

\begin{equation}
   \text{CAC} = \frac{1}{|\mathcal{D}|} \sum_{(\Tilde{x}, \mathcal{W}_{\Tilde{x}}) \in \mathcal{D}} \frac{1}{|\mathcal{W}_{\Tilde{x}}|} \sum_{\Tilde{x}' \in  \mathcal{W}{_{\Tilde{x}}}} \mathbb{1}(g(\Tilde{x}) = g(\Tilde{x}'))
\end{equation}
    

Unlike CAD, CAC also measures the distance of embeddings to those of the opposite class.
This provides a more direct signal of the separability of classes that better correlates with downstream classification performance.

\subsubsection{Gaussian potential uniformity (GPU)} Uniformity measures how evenly  representations are distributed across the unit hypersphere. Wang \etal \cite{DBLP:journals/corr/abs-2005-10242} define uniformity through the logarithm of the average pairwise Gaussian potential:
\begin{equation}
\text{GPU} = \log\left(\frac{1}{\binom{|N|+1}{2}} \sum_{k=1}^{N} \sum_{j=1}^{N} e^{-\|f(\Tilde{x}_k) - f(\Tilde{x}_j)\|^2_2}\right)
\end{equation}

A lower GPU indicates that the embeddings are more evenly spread across the hypersphere. Utilizing large portions of the hypersphere for embeddings suggests a broader range of features learned and thus increased generalizability to unseen data.

\section{Methods}

\begin{figure}[tb]
\centering
    \includegraphics[width=\linewidth]{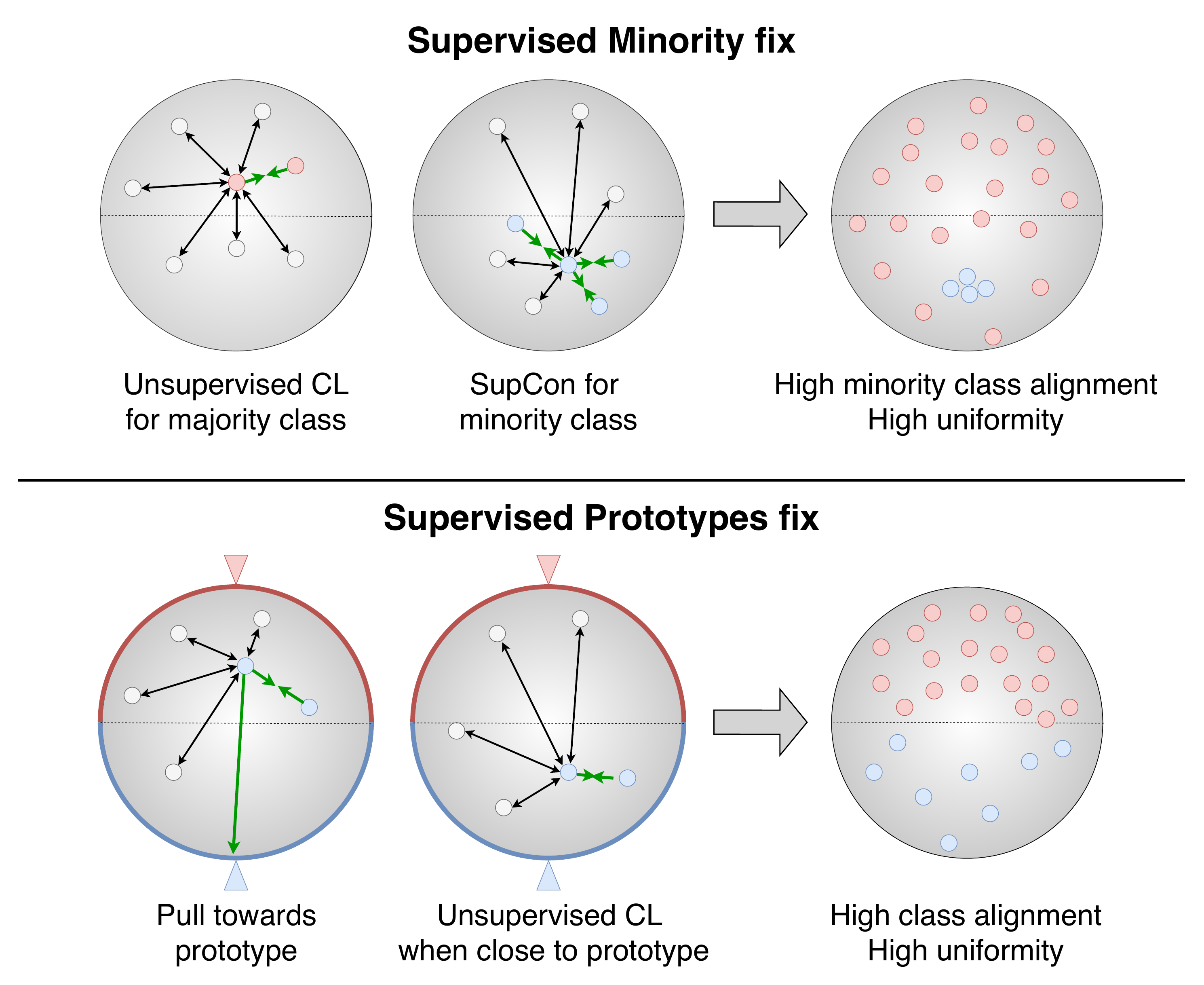}
    \caption{We introduce two fixes for supervised contrastive learning. Supervised Minority applies supervision exclusively to the minority class, preventing class collapse and enhancing alignment of the minority class. Supervised Prototypes attracts samples to fixed class prototypes, improving both class alignment and uniformity.}
    \label{fig:sup_fixes}
\end{figure}

\subsection{SupCon} \label{eq:sup_con_loss}

SupCon is best understood as an extension of the original unsupervised contrastive NT-Xent loss. In essence, the unsupervised loss maximizes the cosine similarity between embeddings of positive pairs while minimizing the similarity between embeddings of negative pairs. Using the notation introduced in \cref{sec:definitions}, the loss is:

\begin{equation}
    \mathcal{L}_{\text{NT-Xent}} = \quad -\sum_{x \in \mathcal{X}} \log \frac{e^{f(\tilde{x}) \cdot f(\tilde{x}^+)/\tau}}{\sum_{\tilde{x}^- \in \mathcal{W} \setminus \{\tilde{x}\}}e^{f(\tilde{x}) \cdot f(\tilde{x}^-)/\tau}}
\end{equation}

\noindent Here, \(\cdot\) denotes the dot product and \(\tau \in \mathbb{R}^+\) is the temperature parameter. In practice, $\mathcal{X}$ and $\mathcal{W}$ are restricted to the elements and views of a single batch.

SupCon extends the NT-Xent loss to include full label information. By considering all samples of the same class in the numerator, it aims to maximize the similarity between all projections of a class. For each class \(i\), the SupCon loss is defined as:

\begin{equation}
\mbox{\footnotesize$
\displaystyle
\mathcal{L}_{\text{SupCon}}^{i} \triangleq \\
-\sum_{x \in \mathcal{X}_i} \frac{1}{|\mathcal{W}_i \setminus \{\Tilde{x}\}|} \sum_{p \in \mathcal{W}_i \setminus \{\Tilde{x}\}} \log \frac{e^{f(\Tilde{x}) \cdot f(p) / \tau}}{\sum_{a \in \mathcal{X} \setminus \{\Tilde{x}\}} e^{f(\Tilde{x}) \cdot f(a) / \tau}}
$} 
\end{equation}
The total loss in the binary case is then:
\begin{equation}
\mathcal{L}_{\text{SupCon}} = \mathcal{L}_{\text{SupCon}}^{0} + \mathcal{L}_{\text{SupCon}}^{1}
\end{equation}

\subsection{Supervised Minority}
\label{app:def:minsup}

We introduce Supervised Minority, a novel supervised contrastive learning strategy specifically for binary imbalanced datasets. 
Supervised Minority applies supervision exclusively to the minority class (see \cref{fig:sup_fixes}). 
Formally, we combine SupCon in the minority (min) class with the NT-Xent \cite{DBLP:journals/corr/abs-2002-05709} loss in the majority (maj) class: 
\begin{equation}
\mathcal{L}_{\text{SupMin}} = \mathcal{L}_{\text{SupCon}}^{\text{min}} + \mathcal{L}_{\text{NT-Xent}}^{\text{maj}}
\end{equation}

By using the NT-Xent loss for the majority class, we aim to guard against class collapse and increase uniformity. Additionally, by using SupCon in the minority class we enhance its alignment. 


\subsection{Supervised Prototypes}
\label{app:def:supproto}

Our second approach, Supervised Prototypes, builds upon the concept of fixed prototypes \cite{li2022targeted,mettes2019hyperspherical,yang2022inducing,kasarla2022maximum}. We initialize two fixed class prototypes at opposite ends of the representation space's hypersphere. Each prototype attracts samples of its respective class (see \cref{fig:sup_fixes}). While prototypes improve class alignment, they can reduce latent space uniformity \cite{li2022targeted}. To mitigate this, we attract samples towards the prototype only if their cosine similarity with that prototype is less than 0.5. When a sample’s representation already has high similarity to its prototype, it is influenced only by the NT-Xent loss.

We place the majority class prototype $p_{maj}$ on the $\mathcal{S}^{d-1}$ unit sphere so that it minimizes the average distance to all encodings of unaugmented training samples. This position is determined through gradient descent. Let $p_{min} = -p_{maj}$ represent the minority class prototype, ensuring maximal separation on the hypersphere from $p_{maj}$.

The loss of a sample and its class prototype is given by:
\begin{equation}
\mathcal{L}_{p_{\Tilde{x}}}^{i} = -\log \frac{e^{f(\tilde{x}) \cdot p_i/\tau}}
{\sum_{\tilde{x}^- \in \mathcal{W} \setminus \{\tilde{x}\}}e^{f(\tilde{x}) \cdot f(\tilde{x}^-)/\tau}}
\end{equation}
The complete contrastive loss with prototype alignment for class $i$, $\mathcal{L}_{\text{SupConProto}}^{i}$, and overall binary supervised contrastive loss with prototype alignment, $\mathcal{L}_{\text{SupConProto}}$, are defined as:
\begin{equation} 
\mbox{\footnotesize$
\displaystyle
\mathcal{L}_{\text{SupConProto}}^{i} \triangleq \sum_{x \in \mathcal{X}_i} \begin{cases}
                                               \left[ \mathcal{L}_{\text{NT-Xent}}(\Tilde{x}) + \mathcal{L}_{p_{\Tilde{x}}}^{(i)} \right] & \text{if } f(\Tilde{x}) \cdot p_i \leq 0.5 \\
                                                \mathcal{L}_{\text{NT-Xent}}(\Tilde{x})& \text{otherwise} 
                                        \end{cases}
$}
\end{equation}
\begin{equation}
\mathcal{L}_{\text{SupConProto}} = \mathcal{L}_{\text{SupConProto}}^{0} + \mathcal{L}_{\text{SupConProto}}^{1}
\end{equation}

\section{Experimental setup}

\subsection{Datasets}

We utilize a total of seven datasets with binary class distributions that can be grouped into two categories: subsets of the iNaturalist21 (iNat21) dataset \cite{van2021benchmarking}, where we artificially control class imbalance, and real-world medical datasets that naturally exhibit binary distributions and class imbalances. Our three subsets of iNat21 comprise plants (oaks and flowering plants), insects (bees and wasps), and mammals (hoved animals and carnivores). For each subset, we fix the class ratio to 50\%-50\%, 95\%-5\%, and 99\%-1\%, while keeping the total number of samples constant. Our real-world medical datasets include a cardiac datasets curated from the UK Biobank population study \cite{sudlow_uk_2015}, two datasets from the medical MNIST collection, BreastMNIST and PneumoniaMNIST \cite{DBLP:journals/corr/abs-2110-14795}, and the FracAtlas dataset \cite{abedeen2023fracatlas}. Additional details about dataset characteristics and preprocessing are provided in supplementary \cref{app:dataset}.

\subsection{Network architecture and training}

In line with prior work and baselines, we use a ResNet-50 image encoder \cite{DBLP:journals/corr/HeZRS15}, and follow established pre-training protocols\cite{DBLP:journals/corr/abs-2002-05709,  DBLP:journals/corr/abs-2106-03719, DBLP:journals/corr/abs-2004-11362}. After pre-training, we fine-tune a linear layer using a balanced subset comprising 1\% of the training data and report accuracy on a balanced test set. As the medical datasets contain far fewer samples, we do not subsample them for fine-tuning and report the receiver operating characteristic area under the curve (AUC). Further information on the training protocols can be found in supplementary \cref{app:exp_setup}.

\subsection{Baselines}

In addition to standard SupCon~\cite{DBLP:journals/corr/abs-2004-11362} and weighted cross-entropy, we have included the five leading supervised contrastive learning methods for long-tailed datasets as baselines: parametric contrastive learning (PaCo)~\cite{DBLP:journals/corr/abs-2107-12028}, $k$-positive contrastive learning (KCL)~\cite{kang2021exploring}, targeted supervised contrastive learning (TSC)~\cite{li2022targeted}, subclass-balancing contrastive learning (SBC)~\cite{hou2023subclass}, and balanced contrastive learning (BCL)~\cite{zhu2022balanced}. We include results for KCL and TSC with 3 and 6 positives to fairly adapt them to the heavily imbalanced binary case. A brief description of each method can be found in \cref{sec:related_works:longtail}. Further details about the setup and tuning of the baselines is included in supplementary \cref{app:baselines}. Additional baselines using classical, non-contrastive strategies to handle class imbalance, such as focal loss, oversampling, and undersampling, can be found in supplementary \cref{app:add_baselines_conventional}.

\section{Results}
\label{sec:supAnalysis}

\begin{table*}[ht]
\centering
\caption{Balanced accuracy of all evaluated methods on three binary natural imaging datasets at varying degrees of class imbalance. We compare standard weighted cross-entropy loss and supervised contrastive learning (top rows) to five baselines for supervised contrastive learning on long-tailed distributions (middle rows) and our two proposed fixes (bottom rows). Supervised Minority strategy does not apply to balanced settings and thus it is not reported there.}
\footnotesize
\begin{tabular}{@{}l ccc ccc ccc@{}}
\toprule
\multirow{2}{*}{Method} & \multicolumn{3}{c}{Plants} & \multicolumn{3}{c}{Insects} & \multicolumn{3}{c}{Animals} \\
\cmidrule(lr){2-4} \cmidrule(lr){5-7} \cmidrule(lr){8-10}
 & 50\% & 5\% & 1\% & 50\% & 5\% & 1\% & 50\% & 5\% & 1\% \\
\midrule
Weighted CE & 81.1 & 61.4 & 60.1 & 82.4 & 63.4 & 62.8 & 70.7 & 61.9 & 57.3 \\
SupCon \cite{DBLP:journals/corr/abs-2004-11362} & 93.7 $\pm$ 0.6 & 56.2 $\pm$ 1.6 & 54.4 $\pm$ 2.0 & 93.3 $\pm$ 0.1 & 62.6 $\pm$ 0.9 & 56.4 $\pm$ 0.1 & 80.8 $\pm$ 1.2 & 54.4 $\pm$ 1.6 & 56.9 $\pm$ 1.8 \\
\midrule
PaCo \cite{DBLP:journals/corr/abs-2107-12028} & 91.5 $\pm$ 0.9 & 59.2 $\pm$ 1.4 & 55.9 $\pm$ 2.2 & 92.4 $\pm$ 2.2 & 66.4 $\pm$ 0.6 & 53.7 $\pm$ 1.2 & 79.2 $\pm$ 1.5 & 65.3 $\pm$ 1.1 & 55.4 $\pm$ 1.6 \\
KCL (K=3) \cite{kang2021exploring}  & 90.6 $\pm$ 0.7 & 87.6 $\pm$ 0.4 & 81.1 $\pm$ 0.3 & 89.8 $\pm$ 0.3 & 81.1 $\pm$ 1.0 & 73.3 $\pm$ 1.4& 81.8 $\pm$ 0.3 & 76.6 $\pm$ 0.6 & 71.2 $\pm$ 0.8 \\
KCL (K=6)  \cite{kang2021exploring}& 94.2 $\pm$ 0.4 & 86.6 $\pm$ 1.4 & 78.6 $\pm$ 0.5& 91.5 $\pm$ 0.6 & 79.8 $\pm$ 1.2 & 69.2 $\pm$ 2.1 & 82.6 $\pm$  0.6 & 75.3 $\pm$ 0.7 & 70.1 $\pm$ 2.0 \\
TSC  (K=3) \cite{li2022targeted}  & 93.4 $\pm$ 0.8 & 88.0 $\pm$ 0.5 & 79.4 $\pm$ 1.1 & 88.2 $\pm$ 0.6 & 81.2 $\pm$ 1.0 & 73.3 $\pm$ 1.6 & 82.4 $\pm$ 2.5 & 76.1 $\pm$ 0.9& 71.2 $\pm$ 0.8 \\
TSC (K=6) \cite{li2022targeted}  & 94.6 $\pm$ 0.4 & 87.5 $\pm$ 0.7 & 80.1 $\pm$ 0.4 & 91.1 $\pm$ 0.5 & 79.8 $\pm$ 1.3 & 71.2 $\pm$ 2.0 & 83.0 $\pm$ 1.3 &  75.2 $\pm$ 2.1  & 72.2 $\pm$ 0.9 \\
SBC \cite{hou2023subclass} & 75.4 $\pm$ 0.9 & 57.2 $\pm$ 2.6 & 55.6 $\pm$ 1.9 & 77.6 $\pm$ 1.6 & 51.8 $\pm$ 2.3 & 54.0 $\pm$ 3.8 & 72.4 $\pm$ 1.1 & 55.1 $\pm$ 1.3 & 56.1 $\pm$ 2.1 \\
BCL \cite{zhu2022balanced} & 94.1 $\pm$ 0.3 & 85.0 $\pm$ 4.4 & 71.3 $\pm$ 2.9 & \textbf{94.5 $\pm$ 0.1} & 80.0 $\pm$ 2.5 & 74.0 $\pm$ 0.1 & \textbf{86.2 $\pm$ 0.2} & 76.5 $\pm$ 0.5 & 60.3 $\pm$ 0.7 \\
\midrule
\rowcolor{Gray}
Sup Minority & --  & \textbf{89.8 $\pm$ 0.6} & \textbf{85.4 $\pm$ 0.5} & -- & \textbf{82.8 $\pm$ 1.1}& \textbf{78.8 $\pm$ 0.8} & -- & 77.9 $\pm$ 0.9 & \textbf{75.3 $\pm$ 0.5}\\
\rowcolor{Gray}
Sup Prototypes &   \textbf{95.1 $\pm$ 0.2}   & 88.7 $\pm$ 0.7 & 83.4 $\pm$ 1.8 & 93.0 $\pm$ 0.3  & 81.2 $\pm$ 0.7 & 73.7 $\pm$ 1.3 & 82.9 $\pm$ 0.7 & \textbf{79.2 $\pm$ 1.3} & 73.0 $\pm$ 1.3 \\
\bottomrule
\end{tabular}
\label{tab:inat21_results}
\end{table*}

First, we measure the performance of SupCon on binary imbalanced distributions, showing that it inversely correlates with dataset imbalance. Next, we show how the newly introduced SAA and CAC can diagnose representation space collapse, an issue that canonical alignment metrics fail to detect. We substantiate these findings by introducing a proof that provides a mathematical explanation for the observed behavior. Finally, we show the benefit of our proposed supervised contrastive learning strategies for binary imbalanced datasets compared to existing baselines for long-tailed distributions.

\subsection{SupCon performance on binary datasets degrades with increasing class imbalance}

We first evaluate SupCon on the three binary natural image datasets while varying the degree of class imbalance (see \cref{tab:inat21_results}). We observe a sharp drop in linear probing accuracy as class imbalance increases. Specifically, models trained with 1\% and 5\% minority class representation achieve downstream accuracies between 50\% and 60\%, compared to over 90\% accuracy in the balanced case. While SupCon outperforms the weighted cross-entropy baseline on balanced datasets by over 10\%, it underperforms this simple baseline by over 5\% on binary imbalanced distributions. We find that the performance of SupCon drops below that of weighted cross-entropy around 20\% imbalance, before completely collapsing between 5\% and 1\% (see supplementary \cref{app:imbalance_vs_performance}).

\begin{figure*}[htbp]
    \centering
    \includegraphics[width=0.9\textwidth]{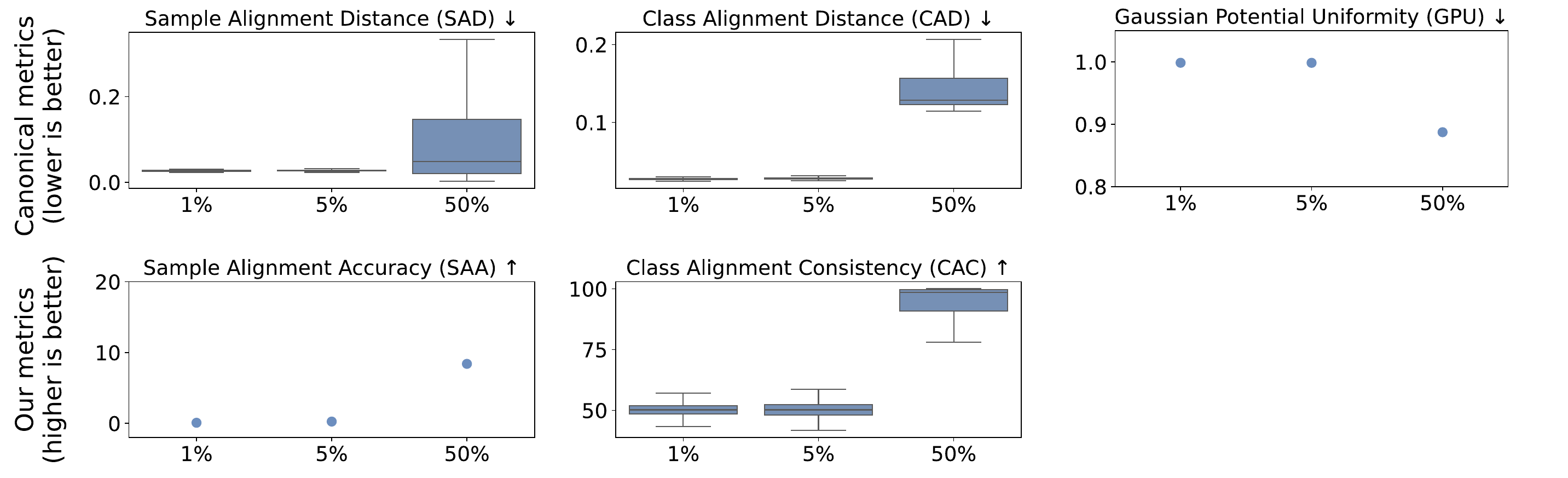}
    \caption{Boxplots of metrics analysing SupCon's representation space learned from the plants dataset. As class imbalance grows the representation space collapses despite the canonical SAD and CAD metrics being low. In contrast, SAA and CAC correctly identify the collapse. Similar results are observed on the insects and animals datasets (see supplementary \cref{app:align_unif}).}
    \label{fig:sup_metrics}
\end{figure*}

\subsection{Analyzing representation spaces using canonical metrics and novel metrics}

To understand the underlying causes of the observed degradation, we analyze the learned representation spaces using both the canonical and new metrics (see \cref{fig:sup_metrics}). Both the canonical SAD and CAD are close to zero across all imbalances, suggesting high alignment of the learned representation space. However, they fail to put the distance in context to samples from other instances or classes. In comparison, our novel SAA and CAA metrics indicate that embeddings do not form distinct class-wise clusters. An SAA of 0 shows that the learned embeddings cannot differentiate between samples by input semantics. A CAC close to 50\% suggests that both minority and majority class samples are almost equally mixed in local neighborhoods. Together, the new metrics confirm that SupCon’s representation space collapses under heavy imbalance.


To further substantiate this empirically observed behavior, we present a proof in supplementary \cref{app:proof}. The proof shows that gradients in the final network layer are upper-bounded by the inverse of the number of positives for a given sample. When the majority class dominates training batches, the gradient quickly saturates, preventing meaningful updates for that class and causing collapse towards a single point in the representation space.


\begin{table*}[ht]
\caption{Area under the curve (AUC) of all evaluated methods on four medical imaging datasets. We compare standard weighted cross-entropy loss and supervised contrastive learning (top rows) to five baselines for supervised contrastive learning for long-tailed distributions (middle rows) and our two proposed fixes (bottom rows).}
\centering
\footnotesize
\begin{tabular}{@{}l cc cc c@{}}
\toprule
\multirow{2}{*}{Method} & \multicolumn{1}{c}{UKBB} & \multicolumn{2}{c}{MedMNIST} & FracAtlas \\
\cmidrule(lr){2-2} \cmidrule(lr){3-4} \cmidrule(lr){5-5}
  & Infarction (4\%) & BreastMNIST (37\%) & PneumoniaMNIST (35\%) & Fractures (21\%)\\
\midrule
Weighted CE & 72.4 & 75.1 &  98.8 & 79.8\\ 
SupCon \cite{DBLP:journals/corr/abs-2004-11362} & 61.9 $\pm$ 1.9 & 75.1 $\pm$ 0.7 & 99.5 $\pm$  0.1 &84.8 $\pm $0.1\\
\midrule
PaCo \cite{DBLP:journals/corr/abs-2107-12028} & 66.6 $\pm$ 1.3 & 66.0 $\pm$ 1.9 & 98.7 $\pm$ 0.2 & 83.7 $\pm$ 0.6 \\
KCL (K=3) \cite{kang2021exploring}          &  75.3 $\pm$ 0.4 & 89.9 $\pm$ 0.8 & 99.6 $\pm$ 0.1 &  \textbf{88.2 $\pm $ 0.1}\\
KCL (K=6) \cite{kang2021exploring}          & 73.6 $\pm$ 0.2 & 89.5 $\pm$ 0.4 & 98.9 $\pm$ 0.1 & 86.5 $\pm $ 0.1 \\
TSC (K=3) \cite{li2022targeted}             & 75.7 $\pm$ 0.1 & 89.2 $\pm$ 0.1 &99.5 $\pm$ 0.1 & 87.1 $\pm$ 0.1 \\
TSC (K=6) \cite{li2022targeted}             & 75.0 $\pm$ 0.1 & 88.5 $\pm$ 0.1 & 99.5 $\pm$ 0.1 & 86.3 $\pm$ 0.1\\
SBC \cite{hou2023subclass}                  & 70.0 $\pm$ 0.3 & 80.8 $\pm$ 0.7 & 99.3 $\pm$ 1.2 & 80.9 $\pm$ 5.3\\
BCL \cite{zhu2022balanced}                  & 74.0 $\pm$ 0.1 & 90.5 $\pm$ 0.1 & 99.6 $\pm$ 0.1 & 84.9 $\pm$ 0.1\\
\midrule
\rowcolor{Gray}
Sup Minority                                &  77.7 $\pm$ 1.1 & 86.4 $\pm$ 0.2 & 99.6 $\pm$ 0.1 & 82.3 $\pm $0.7\\
\rowcolor{Gray}
Sup Prototypes                              & \textbf{77.9 $\pm$ 0.4}  & \textbf{90.7 $\pm$ 0.5} & \textbf{99.8 $\pm$ 0.1} & 86.0 $\pm$ 0.1\\
\bottomrule
\end{tabular}
\label{tab:medical_results}
\end{table*}

\subsection{Performance of fixes on natural and medical imaging datasets}

Next, we compare our two fixes, Supervised Minority and Supervised Prototypes, against five established baselines for long-tailed supervised contrastive learning. Our Supervised Minority fix achieves the best linear probing performance across all iNat21 datasets (see \cref{tab:inat21_results}), outperforming SupCon by 20\% to 35\%.

Our Supervised Minority fix also surpasses the performance of the five baselines developed for long-tailed datasets. Compared to these, its effectiveness increases at higher class imbalance. At 5\% class imbalance it outperforms all baselines by at least 1\%, and at 1\% imbalance by a margin of 3\%. Supervised Prototypes performed second best on all natural imaging datasets.

Supervised Prototypes performed best on three of the four medical datasets (see \cref{tab:medical_results}). 
Both of our fixes generally match or outperform all five baselines developed for long-tailed data distributions.
On the infarction dataset which has the strongest imbalance (4\%) we see the largest gain of +2\% AUC over the best performing baseline.

Extensive ablations across temperatures, batch sizes and varying degrees of supervision in both the minority and majority class can be found in supplementary \cref{app:ablations}.
Visualizations of the learned embeddings via UMAP in supplementary \cref{app:umap_supervised} also corroborate that Supervised Minority and Supervised Prototypes avoid the representation collapse observed in standard SupCon.

\subsection{Our proposed metrics correlate with downstream classification performance}

\begin{figure*}[htbp]
    \centering
    \includegraphics[width=1.0\textwidth]{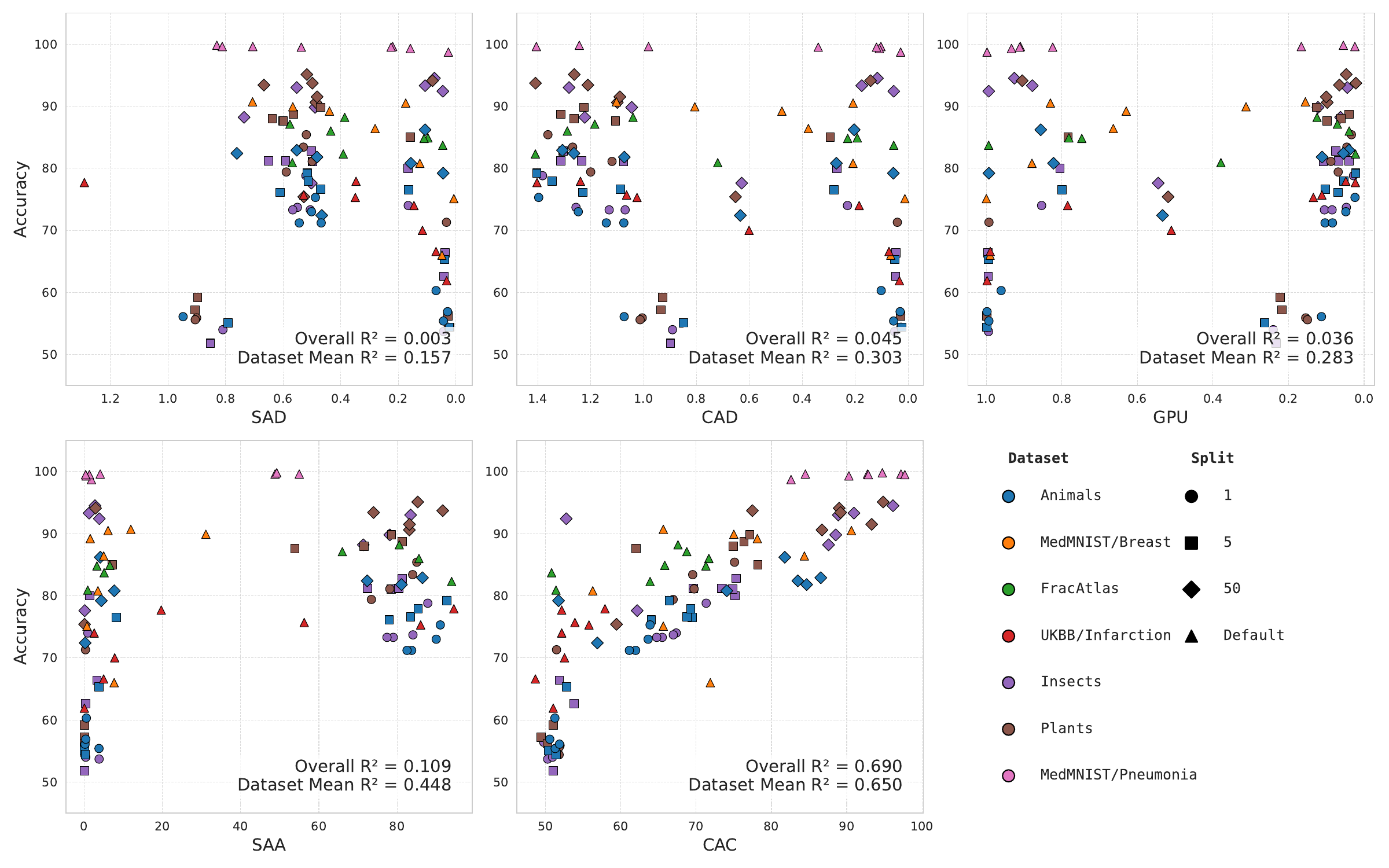}
    \caption{Correlations between five representation-space metrics and linear probing performance across all datasets and all considered methods. The overall $R^2$ is calculated globally over all points while the dataset mean $R^2$ is calculated per dataset and then averaged. As SAA and CAC are the only metrics that account for relationships between samples and classes instead of simply within them, they correlate much stronger with downstream performance.}
    \label{fig:cac_correlation}
\end{figure*}

Classical uniformity (GPU) and alignment (SAD) metrics focus on pairwise sample-level relationships without considering class-level context. While class alignment distance (CAD) incorporates intra-class similarity for supervised contrastive learning, it does not account for the critical inter-class relationships that influence downstream separability.

As shown in \cref{fig:cac_correlation}, the canonical metrics exhibit near-zero correlation with linear probing accuracy when plotted globally across all datasets. 
In contrast, our class alignment consistency achieves an R\textsuperscript{2} value of 0.69, indicating a strong global linear relationship with downstream classification accuracy.

When considering each dataset individually and averaging the correlations over all datasets, we find weak correlations between 0.15 and 0.3 for the classical metrics.
In contrast, our novel metrics achieve a mean $R^2$ value of 0.45 and 0.65 due to the fact that they consider inter-sample and inter-class statistics.
This shows the suitability of our metrics for evaluating representation space quality for downstream utility both on a global and local scale.


\section{Discussion and conclusion}

Although imbalanced binary distributions are commonly found in real-world machine learning problems and especially in medical applications, they have received little attention in the context of supervised contrastive learning. 
In extensive experiments on seven natural and medical imaging datasets, we have shown that SupCon on binary imbalanced distributions often results in collapsed representations, leading to poor downstream performance.

To diagnose these failures, we introduced two new metrics: Sample Alignment Accuracy (SAA) and Class Alignment Consistency (CAC) which extend the notion of alignment to measure how well samples and classes are distinguished from each other in the learned space.
These metrics uncovered shortcomings that canonical measures overlooked, showing that SupCon fails to form meaningful embeddings under high imbalance.
Crucially, CAC correlates well with linear probing accuracy and is thus a suitable metric for measuring representation space quality for downstream applications.

Finally, we proposed two fixes, \emph{Supervised Minority} and \emph{Supervised Prototypes}, specifically tailored to address binary imbalance. Both solutions boost accuracy by up to 35\% over standard SupCon and surpass existing methods for long-tailed distributions by up to 5\%. With minimal additions to the standard SupCon loss and negligible computational overhead, these fixes offer a straightforward path to improved performance on binary imbalanced classification problems.

\paragraph{Limitations}

A limitation of our methods is that the Supervised Minority fix cannot be used when data is balanced (as there is no minority class) and there does not seem to be a particularly clear pattern when Supervised Minority performs better than Supervised Prototypes or vice versa. We observed that Supervised Prototypes always achieved the best performance on medical datasets, while Supervised Minority usually performs better on natural image datasets. We hypothesize that this could either be due to the domain-specific data characteristics or a dependence of both methods on the degree of class imbalance that slightly differed in our experiments. When choosing an approach for a new dataset, it would be prudent to test both methods.

\paragraph{Conclusion} Our study complements and extends a series of previous works that have aimed to explore the theoretical foundations of contrastive learning
and researched its application to long-tailed datasets.
By focusing on the particularly challenging case of binary imbalanced datasets, we have improved the understanding of the dynamics of contrastive learning and developed tools to diagnose and enhance methods dealing with such datasets, which are very common in real-world applications, such as medicine. 

\paragraph{Acknowledgments}
This research has been conducted using the UK Biobank Resource under Application Number 87802.
This work was supported in part by the European Research Council grant Deep4MI (Grant Agreement no. 884622). Martin J. Menten is funded by the German Research Foundation under project 532139938.

%

{
    \small
    \bibliographystyle{ieeenat_fullname}
    \bibliography{main}
}

\FloatBarrier
\clearpage

\clearpage
\setcounter{figure}{0}
\setcounter{table}{0}
\setcounter{section}{0}

\makeatletter
\renewcommand \thesection{S\@arabic\c@section}
\renewcommand\thetable{S\@arabic\c@table}
\renewcommand \thefigure{S\@arabic\c@figure}
\makeatother

\section{Datasets}
\label{app:dataset}

Real-world imbalanced datasets lack the controlability needed to study different levels of imbalance systematically. Therefore, we initially employed a highly controllable artificial dataset for preliminary experiments, only to then validate the results on real-world medical datasets that naturally exhibit binary distributions and class imbalance.

The artificial datasets were derived from the iNaturalist21 dataset \cite{van2021benchmarking}. We split this dataset into binary subsets based on taxonomic ranks and sub-sample them with varying levels of class imbalance. After identifying representation space issues and developing solutions using the controlled datasets, we validated our approach on four real-world medical datasets with natural binary class distributions: PneumoniaMNIST and BreastMNIST from Medical MNIST \cite{DBLP:journals/corr/abs-2110-14795}, a cardiac dataset from the UK Biobank \cite{sudlow_uk_2015} and FracAtlas \cite{abedeen2023fracatlas}.

\subsection{Artificially imbalanced datasets (iNat21)}

Dataset selection is critical in (supervised) contrastive learning, as class semantics directly influence the learning process. High intra-class similarity (intra-class homogeneity) enhances the learning of discriminative features within classes, while low inter-class similarity (inter-class heterogeneity) aids in distinguishing between classes \cite{DBLP:journals/corr/abs-2008-10150, wang2022chaos}.  Tsai et al. \cite{DBLP:journals/corr/abs-2006-05576} emphasized the need for latent classes to embody task-relevant information within the training data.

To cover these effects, we select subsets of the iNaturalist 2021 (iNat21) dataset with different levels of homogeneity in and between classes. We determine semantic homogeneity between classes by using the hierarchical taxonomy, measuring class distances by steps in the taxonomy tree. Within-class heterogeneity is assessed based on the number of subspecies, taxonomic rank, and visual similarities of species, habitats, and backgrounds.

Based on these criteria, we selected three class categories:

$$
\begin{array}{|l|c|c|}
\hline
\textbf{Dataset} & \textbf{Intra-class} & \textbf{Inter-class} \\
\hline
\text{Plants} & \text{Mixed*} & \text{Heterogeneous} \\
\text{Insects} & \text{Homogeneous} & \text{Homogeneous} \\
\text{Animals} & \text{Heterogeneous} & \text{Heterogeneous} \\
\hline
\end{array}
$$
*One class homogeneous, one class heterogeneous

\subsubsection{Plants dataset (asymmetric)}

\begin{figure}[ht]
    \centering
    \begin{subfigure}[b]{\linewidth}
        \includegraphics[width=\textwidth]{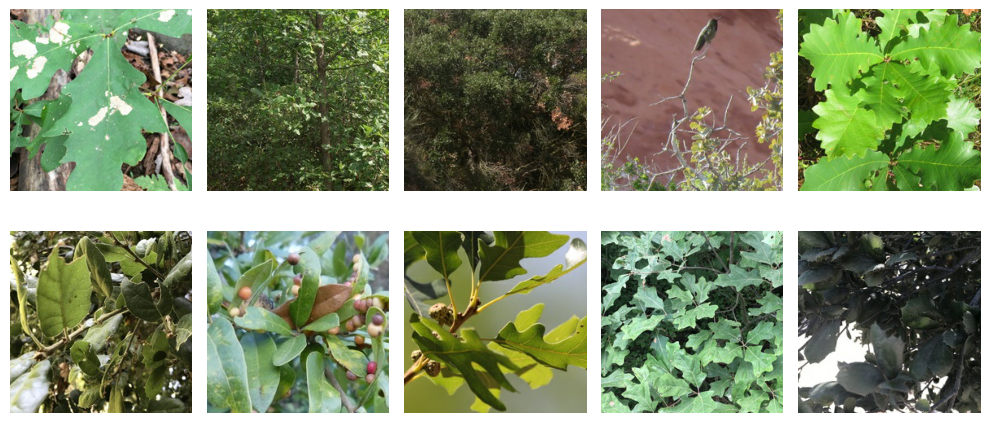}
        \caption{The homogeneous genus Quercus (oak)}
        \label{fig:oak_subfig}
    \end{subfigure}
    \begin{subfigure}[b]{\linewidth}
        \includegraphics[width=\textwidth]{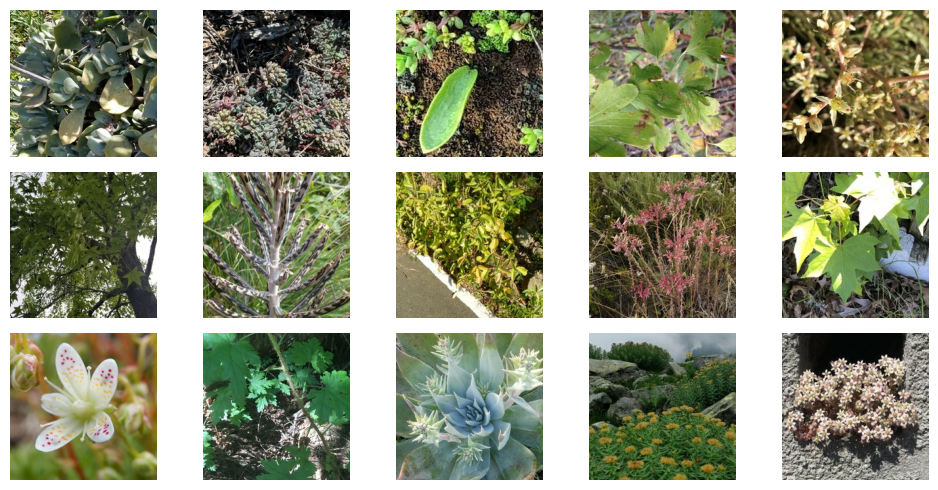}
        \caption{The heterogeneous Saxifragales order}
        \label{fig:saxi_subfig}
    \end{subfigure}
    \caption{The plants dataset illustrated by the homogeneous Quercus genus with its visually similar leaves and trees, and contrasting with the Saxifragales order, which exhibits high inner-class heterogeneity with a diverse array of plant forms from flowers to cacti, bushes, and trees.}
    \label{fig:asymmetric_dataset}
\end{figure}

The dataset presents a clear contrast between its homogeneous and heterogeneous classes, marked by both inner-class characteristics and high between-class heterogeneity. Within this dataset, the Quercus genus, categorized under the taxonomy \textit{Plantae} $\rightarrow$ \textit{Tracheophyta} $\rightarrow$ \textit{Magnoliopsida} $\rightarrow$ \textit{Fagales} $\rightarrow$ \textit{Fagaceae} $\rightarrow$ \textit{Quercus} (\cref{fig:oak_subfig}), represents a homogeneous class with 11,785 instances across 43 species. This class is characterized by low inner-class heterogeneity, exhibiting minimal variance within the class, with visually similar leaves and trees.

In contrast, the axifragales order, following the taxonomy \textit{Plantae} $\rightarrow$ \textit{Tracheophyta} $\rightarrow$ \textit{Magnoliopsida} $\rightarrow$ \textit{Saxifragales} (\cref{fig:saxi_subfig}), serves as the heterogeneous class with 21,641 instances spanning 82 species. This class encompasses a wide variety of plant forms, including trees, shrubs, herbs, succulents, and aquatic plants, contributing to its high inner-class heterogeneity.

The distance between classes in the taxonomy tree is small, as Saxifragales is an order and thus two levels higher in the hierarchy than Quercus, a genus. This contrasts with the significant differences in class diversity and characteristics, emphasizing the dataset's asymmetry.

\subsubsection{Insects dataset (homogeneous)}
\begin{figure}[ht]
    \centering
    \begin{subfigure}[b]{1\linewidth}
        \includegraphics[width=\textwidth]{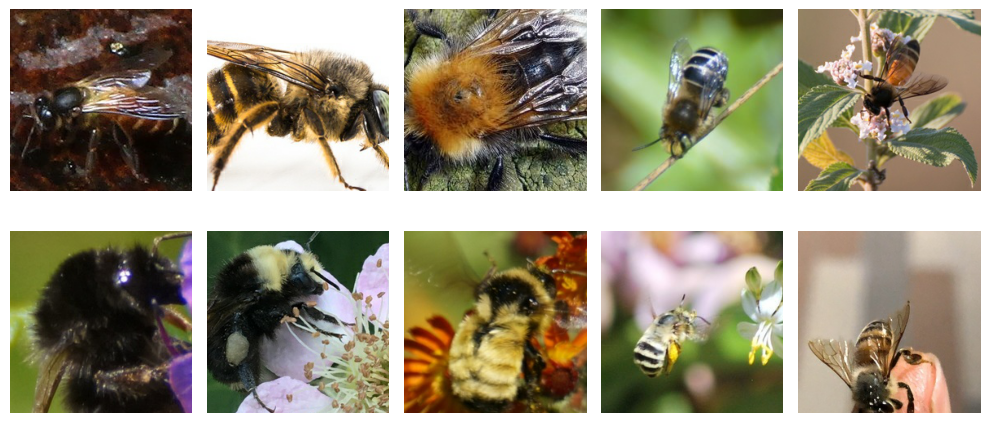}
        \caption{Apidae class (bees)}
        \label{fig:bees_subfig}
    \end{subfigure}
    \hfill
    \begin{subfigure}[b]{1\linewidth}
        \includegraphics[width=\textwidth]{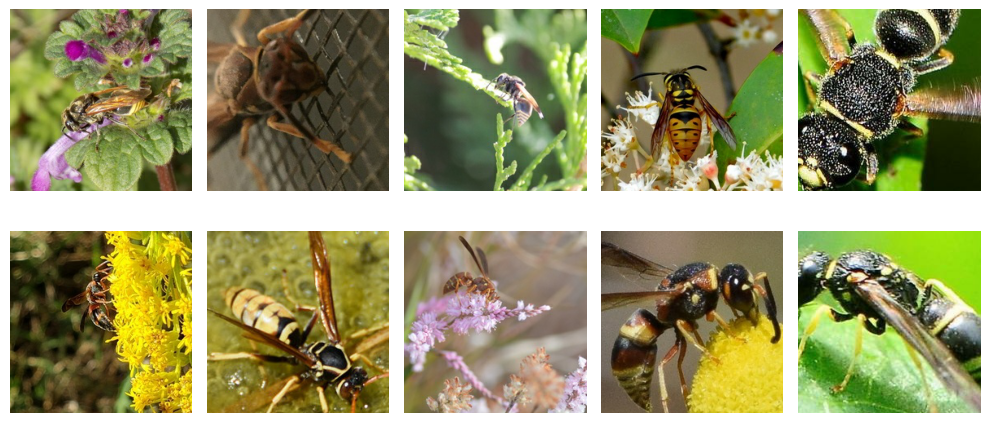}
        \caption{Vespidae class (wasps)}
        \label{fig:wasps_subfig}
    \end{subfigure}
    \caption{Representative images from the insects (homogeneous) dataset showcasing the two closely related classes, Apidae (bees) and Vespidae (wasps), exemplifying the dataset's homogeneity. Both classes demonstrate consistent visual characteristics, high intra-class homogeneity, and maintain a short taxonomic branch distance, underlining their similarities while retaining distinct biological traits.}
    \label{fig:combined_homogeneous_classes}
\end{figure}
This dataset comprises two closely related and homogeneous classes, Apidae (bees, \cref{fig:bees_subfig}) and Vespidae (wasps, \cref{fig:wasps_subfig}). These classes are neighbors in the taxonomy tree with a branch distance of two, both belonging to the hierarchy level of family. They display similarities in species count, sample numbers, and visual characteristics, including consistent backgrounds in photography.

The Apidae family, which consists mainly of bees, is represented by 11,740 samples spanning 38 species. The Vespidae family comprises 9,929 samples distributed across 42 species. Both families share a common taxonomic hierarchy, underlining their similarities while retaining distinct biological traits.

\subsubsection{Mammals dataset (heterogeneous)}

\begin{figure}[ht]
    \centering
    \begin{subfigure}[b]{\linewidth}
        \includegraphics[width=\textwidth]{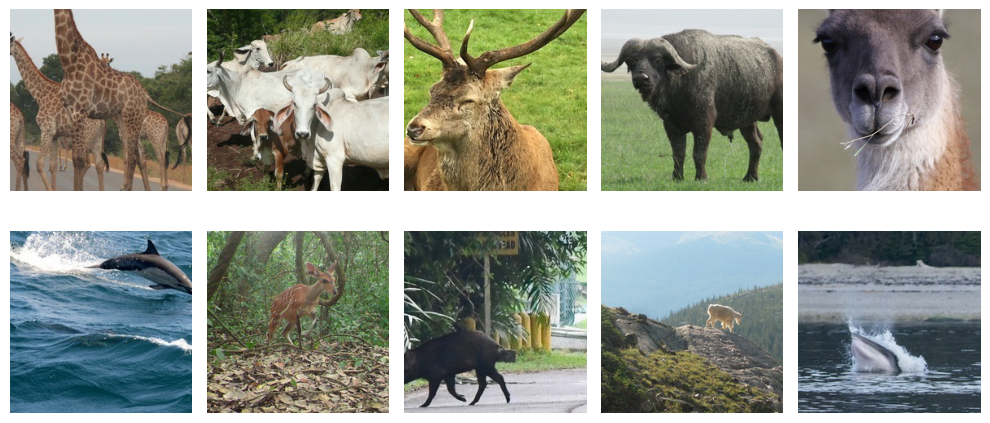}
        \caption{Artiodactyla class}
        \label{fig:artiodactyla_subfig}
    \end{subfigure}
    \begin{subfigure}[b]{\linewidth}
        \includegraphics[width=\textwidth]{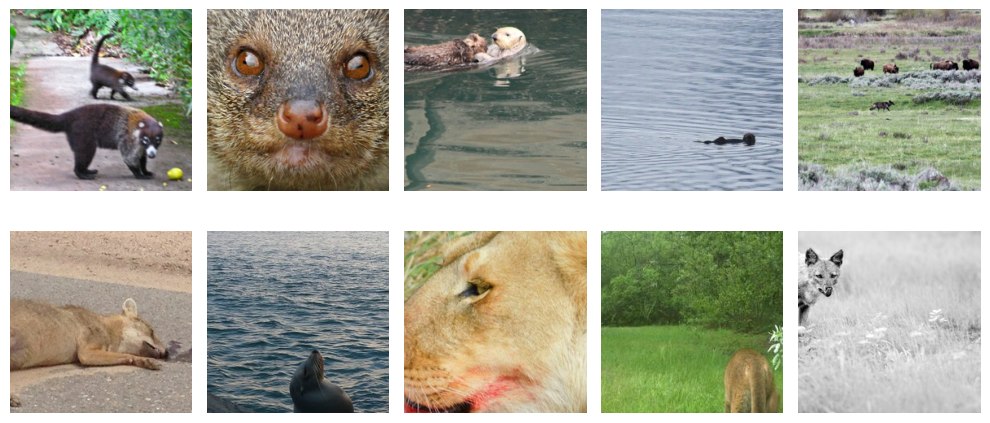}
        \caption{Carnivora order}
        \label{fig:carnivora_subfig}
    \end{subfigure}
    \caption{Representative images from the animals (heterogeneous) dataset illustrating heterogeneity through examples from the diverse Artiodactyla class and the Carnivora order. The Artiodactyla exemplify inner-class diversity with species ranging from dolphins to giraffes and bovines, while the Carnivora, showing a similar diversity, includes species such as lions, ferrets, and sea lions. These images underscore the dataset's broad spectrum of biological diversity.}
    \label{fig:combined_heterogeneous_classes}
\end{figure}

This dataset focuses on two highly diverse classes of mammals, Artiodactyla (\cref{fig:artiodactyla_subfig}) and Carnivora (\cref{fig:carnivora_subfig}), both of which demonstrate significant inner-class and between-class heterogeneity. The Artiodactyla order, classified under Animalia, comprises 15,917 samples across 54 species. This group includes a wide range of species, such as deer, antelopes, bovines, dolphins, and giraffes, each with distinct morphological traits.

 Similarly, the Carnivora order, contains 15,360 samples distributed among 55 species. 
 This class encompasses predators and omnivores like bears, felines, canines, ferrets, and sea lions.

Artiodactyla species differ significantly from those in Carnivora, living in different ecological habitats and exhibiting a wide range of physical characteristics, highlighting the dataset's high between-class heterogeneity.

\subsection{Dataset splits}
We sub-sampled our datasets to enable training across any split ranging from 1\% to 99\% for both classes while maintaining a constant total sample size across all experiments. To achieve this, we initially downsampled the more populous class to match the size of the smaller class before artificially imbalancing the two (see \cref{tab:example_imabalance_count} for exact numbers).
\begin{table*}[ht]
    \centering
    \begin{tabular}{l||c|c|c|c|c|c}
    
     Imbalance Ratio & Total Samples & $1\% : 99\%$  & $5\% : 95\%$ &  $50\% : 50\%$ \\
         \toprule
         Heterogeneous Dataset & $14,577$ & $145 : 14,432$ & $728 : 13,849$ & $7,289 : 7,289$ \\
         Homogeneous Dataset & $9,438$ & $94 : 9,344$ & $471 : 8,967$ & $4,719 : 4,719$ \\
         Asymmetric Dataset & $11,197$ & $111 : 11,086$ & $559 : 10,638$ & $5,599 : 5,599$
    \end{tabular}
    \caption{Distribution of samples across various levels of dataset imbalance. The table provides the count of samples for both classes in each scenario for heterogeneous, homogeneous, and asymmetric datasets. Test and validation sets are always balanced, ensuring valid comparisons between the splits.}
    \label{tab:example_imabalance_count}
\end{table*}
\subsection{Medical datasets}
\subsubsection{UK Biobank cardiac data}

Our first medical dataset originates from the UK Biobank, a comprehensive biomedical database containing genetic and health data from over 500,000 UK individuals \cite{sudlow_uk_2015}.
We used short-axis cardiovascular magnetic resonance (CMR) imaging data, originally comprising 46,656 subjects, each with a 4D MRI image stack. For our experiments, we utilized the middle slice of three time-points (End-Systolic, Mid-Systolic, and End-Diastolic) from each stack, which we encoded as an image's three channels. This dataset features class imbalances of 0.035 for infarction vs. rest and 0.086 for coronary artery disease (CAD) vs. rest. The labels were generated using the hospital admission ICD codes of the patients and include both past and future diagnoses. This was done to account for the fact that many cardiovascular diseases go undiagnosed for years until a severe event brings the patient into the hospital \cite{nesto1999screening,valensi2011prevalence,Hager_2023_CVPR}.

\begin{figure}
    \centering
    \includegraphics[width=\linewidth]{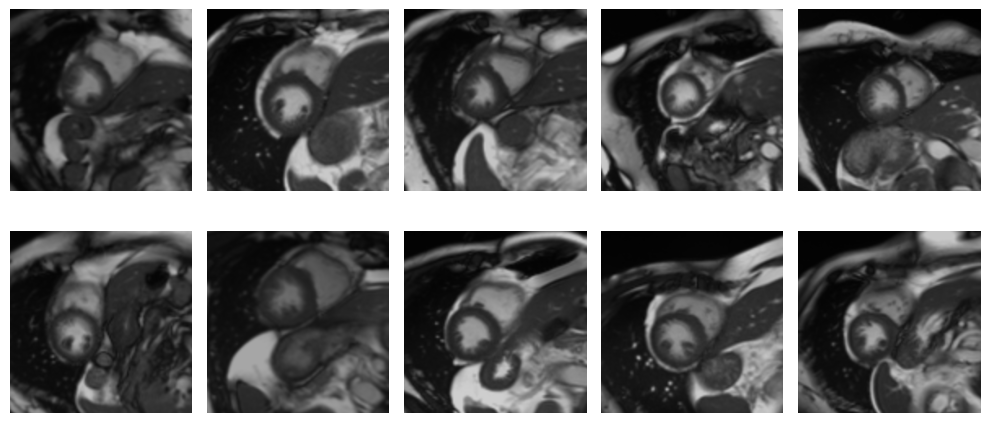}
    \caption{UKBB Cardiac}
    \label{fig:example_images_ukbb}
\end{figure}
\subsubsection{MedMNIST data}

MedMNIST\cite{DBLP:journals/corr/abs-2110-14795} provides standardized datasets for biomedical image classification with multiple size options: 28 (MNIST-Like), 64, 128, and 224 pixels. We chose the 224-pixel size and selected two datasets that naturally exhibit binary distributions. We use the original train, test and validation splits.

\paragraph{PneumoniaMNIST}
Derived from pediatric chest X-ray images, this dataset is used for binary classification of pneumonia with a class imbalance of 0.35 (positive) vs. 0.65 (negative).
See~\cref{fig:example_images_pneumonia} for some example images.
\begin{figure}
    \centering
    \includegraphics[width=\linewidth]{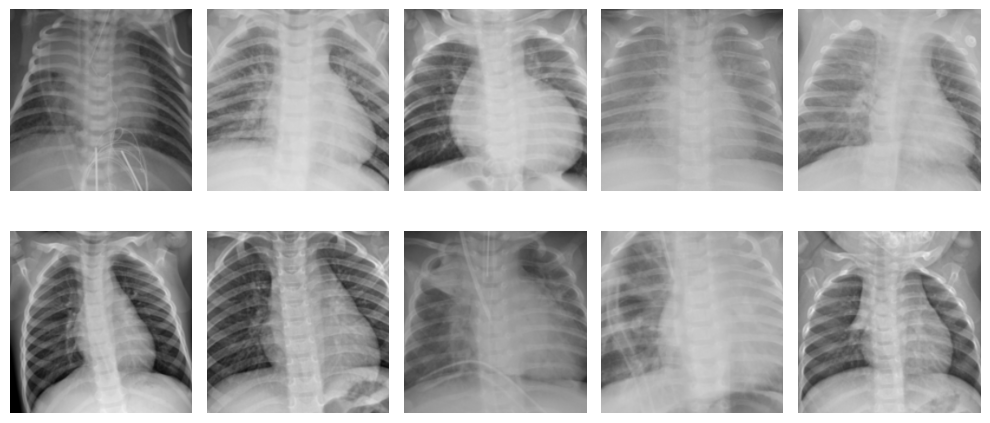}
    \caption{PneumoniaMNIST}
    \label{fig:example_images_pneumonia}
\end{figure}

\paragraph{BreastMNIST}
Sourced from breast ultrasound images, this dataset categorizes images into normal and benign (grouped as positive) and malignant (negative) with an imbalance of 0.368 (positive) vs. 0.632 (negative). See~\cref{fig:example_images_breast} for some example images.

\begin{figure}
    \centering
    \includegraphics[width=\linewidth]{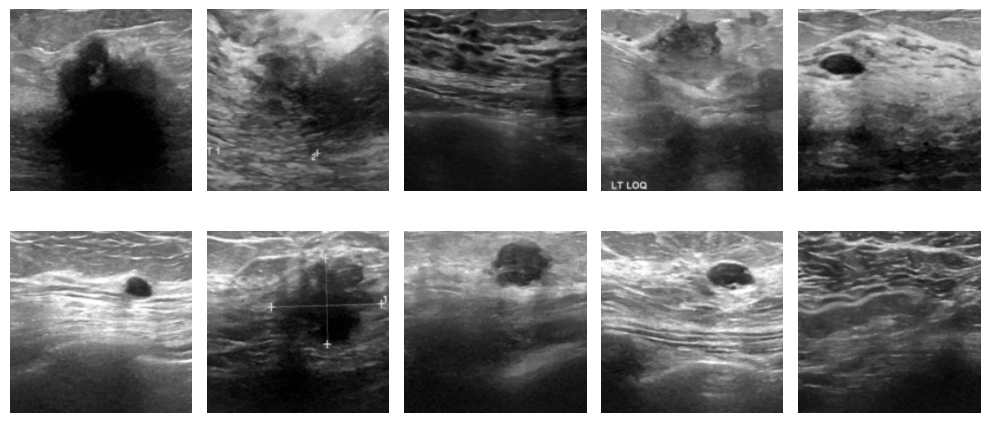}
    \caption{BreastMNIST}
    \label{fig:example_images_breast}
\end{figure}

\subsubsection{FracAtlas Dataset}

The FracAtlas dataset is a collection of medical imaging data focusing on bone fractures, published in Nature Scientific Data \cite{abedeen2023fracatlas}. It includes 4,024 X-ray images annotated by medical professionals, covering different types of fractures across multiple anatomical locations such as the femur, tibia, humerus, radius, and others. The dataset features a class imbalance representative of clinical settings, with approximately 0.21 (fractured) vs.\ 0.79 (non-fractured), reflecting the lower proportion of fracture cases compared to normal cases typically seen in clinical practice. We use a 80\%:10\%:10\% train, validation, and test splits. See~\cref{fig:example_images_frac_atlas} for some example images.

\begin{figure} 
\centering 
\includegraphics[width=\linewidth]{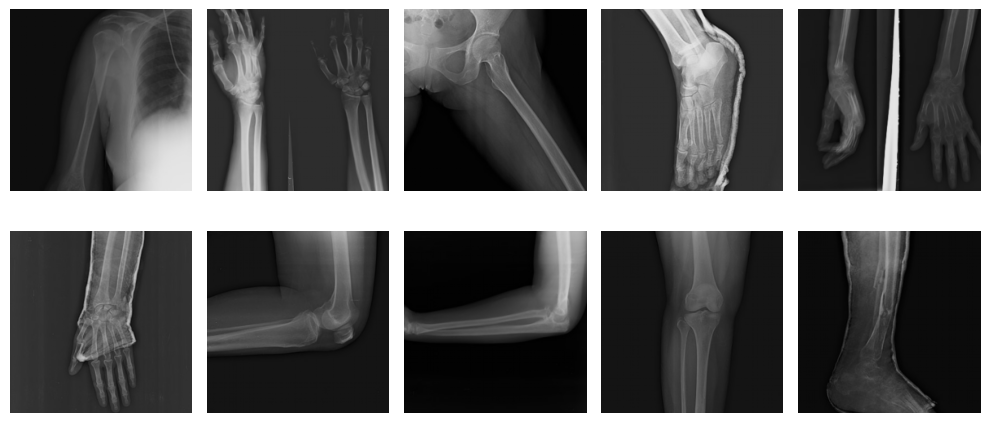} \caption{FracAtlas Dataset} \label{fig:example_images_frac_atlas} \end{figure}

\section{Experimental setup}
\label{app:exp_setup}

\subsection{Contrastive pre-training details}
All experiments in the main paper were conducted using a ResNet-50 backbone \cite{DBLP:journals/corr/HeZRS15}. The augmentations and linear projection head were adapted from the original SimCLR paper with a projection dimension of 128 \cite{DBLP:journals/corr/abs-2002-05709}. We set the batch size to 256, ensuring that at least one anchor of the minority class is included on average even at 1\% imbalance ratio. Training was conducted for 350 epochs. 

For optimization, we used stochastic gradient descent (SGD) with a momentum of 0.9 and a weight decay of $1 \times 10^{-4}$. A cosine annealing learning rate scheduler with a 10-epoch warm-up period \cite{DBLP:journals/corr/LoshchilovH16a} was utilized, starting with a learning rate of 0.00625 and warming up to 0.0625. We use a temperature value of 0.07. To ensure fairness, all approaches were trained for the same number of epochs and with the same backbone architecture. For all medical datasets, the training epochs were reduced to 250.

\subsubsection{Supervision in minority loss}

Training details are consistent with the setup described above. The only modification is the loss function, as described in \cref{app:def:minsup}.

\subsubsection{Supervised prototype loss}

Training details follow the same setup outlined above. The only difference is the loss function, detailed in \cref{app:def:supproto}.

\subsection{Weighted cross-entropy training details}

As a baseline comparison, we employed weighted cross-entropy to counteract class imbalance effects. The model was optimized using Adam \cite{kingma2017adammethodstochasticoptimization} with an initial learning rate of $1 \times 10^{-4}$ and no dropout or weight decay applied. The weight used for each class was the inverse of its frequency.

\subsection{Evaluation details}

To evaluate the quality of the learned representations, a linear probing protocol using a single linear layer was followed \cite{DBLP:journals/corr/abs-2002-05709,DBLP:journals/corr/abs-2106-03719,DBLP:journals/corr/abs-2004-11362}. All pre-trained encoder weights were frozen and the linear head was trained for 50 epochs using only resize and center crop augmentations. A subset of 1\% of the balanced pretraining dataset was used for linear probing along with a constant learning rate of $3 \times 10^{-4}$and the SGD optimizer with momentum of 0.9 and weight decay of $1 \times 10^{-4}$.

\subsection{Augmentations}
\subsubsection{iNat21}
For the iNat21 dataset, we used standard SimCLR transformations \cite{DBLP:journals/corr/abs-2002-05709}. These include random resized cropping to 224$\times$224 pixels and random horizontal flipping with a probability of 0.5. Additionally, we applied color jittering with a probability of 0.8 and random grayscaling with a probability of 0.2. Finally, we applied z-normalization to the images.

\subsubsection{MedMNIST \& FracAtlas}

For the MedMNIST and FracAtlas datasets, we applied augmentations specifically designed for grayscale medical images. Single-channel images were replicated across three channels. Subsequently, images were randomly cropped to a scale range of 25\% to 100\% of the original image area, with an aspect ratio range from 0.75 to 1.33, and then resized to the target size. Random horizontal flipping was applied with a probability of $p=0.5$, along with selective color jittering to adjust brightness and contrast within a range of $\pm15\%$ and a probability of $p=0.8$. Finally, images were z-normalized as part of the transformation process.

\subsubsection{UKBB Cardiac}

For the UKBB cardiac dataset, we used a combination of random horizontal flipping with a probability of $p=0.5$ and random rotations up to 45 degrees. Color adjustments were applied to jitter brightness, contrast, and saturation within a range of $\pm50\%$ and $p=0.8$. Additionally, images were randomly resized with a scale range of 20\% to 100\% of the original image size, cropped to 128 pixels, and finally z-normalized.

\subsubsection{Evaluation transforms}
For evaluation purposes, images were first resized to 256 pixels and then center-cropped to 224 pixels and z-normalized to maintain a consistent aspect ratio and size across all datasets.

\section{Baselines}
\label{app:baselines}
\subsection{\texorpdfstring{$k$}{k}-Positive Contrastive Learning (KCL)}

In the $k$-Positive Contrastive Learning (KCL) method, we draw $k$ instances from the same class to form the positive sample set. While the original paper by Kang \etal \cite{kang2021exploring} sets $k=6$, we also benchmarked with $k=3$ due to the pronounced class imbalances in our dataset. For strong imbalances, in some batches, there are not enough positive samples for the majority class, averaging only $\lceil2.56\rceil =3$. As demonstrated in the results section, we find that $k=3$ is more effective for heavy imbalances in the binary case.  We implemented the KCL loss directly in our pipeline and used the same hyperparameters as described above~\ref{app:exp_setup}.

Although KCL is only briefly described in the appendix and named differently, it was also mentioned in the original Supervised Contrastive Learning (SupCon) paper \cite{DBLP:journals/corr/abs-2004-11362}.

\subsection{Targeted Supervised Contrastive Learning for Long-Tailed Recognition (TSC)}

TSC extends KCL by introducing class prototypes. Instead of using the MoCo implementation from the authors' repository, we implemented the described loss within our SupCon framework for better comparability with other SupCon variations. We set the hyperparameter $\lambda = 1$ which weights the contribution of the prototypes to the total loss. Lambda is not specified in the original paper but was inferred from the authors' code. The remaining hyperparameters were set as above. We tested both $k=3$ and $k=6$ in our experiments.

\subsection{Balanced Contrastive Learning for Long-Tailed Visual Recognition (BCL)}
While TSC learns targets without explicit class semantics, BCL leverages class prototypes as additional samples. The BCL framework consists of a classification branch and a balanced contrastive learning branch, sharing a common backbone. The classifier weights are transformed by an MLP to serve as prototypes. We standardized the data augmentations with those used in other baselines for a fair comparison. The learning rate, batch size, and other training hyperparameters are identical to those described above.

\subsection{Subclass-Balancing Contrastive Learning for Long-Tailed Recognition (SBC)}

We utilized the original authors' code for SBC but replaced the backbone with the same ResNet architecture used in all our experiments. The original class imbalance was maintained as the imbalance factor. We removed the warm-up period during which only SupCon is applied, as our experiments have shown that SupCon collapses for our data. Clusters were updated every 10 epochs, as suggested by the authors' code. We employed the "train rule rank" with a ranking temperature of 0.2 and used $grama=0.25$ (grama in their code is called called $\beta$ in their paper \cite{hou2023subclass}), following the authors' recommendations. The rest of the hyperparameters are as described previously.

\subsection{Parametric contrastive learning baseline}

Paco is a MoCo-based \cite{DBLP:journals/corr/abs-1911-05722} strategy which we did not re-implement to SimCLR, considering momentum is a crucial component of the loss function. The parametric contrastive learning \cite{DBLP:journals/corr/abs-2107-12028} baseline follows the hyperparameter suggestions of the original paper, setting the alpha parameter ($\alpha$) to 0.05, the beta ($\beta$) and gamma ($\gamma$) parameters, which control the weighting of various losses, were both set to 1.0. Weight decay was set to $1 \times 10^{-4}$. The learning rate for this baseline was set at 0.0625. We used a MoCo-t temperature of 0.2, the MoCo queue size (MoCo-k) of 8192, and a MoCo embedding dimension (MoCo-dim) of 128. The momentum for the moving average encoder (MoCo-m) was set to 0.999.

\section{Additional baselines on traditional data imbalance strategies}
\label{app:add_baselines_conventional}

\begin{table}[h]
\centering
\footnotesize
\begin{tabular}{|l|c|c|c|c|c|c|}
\hline
& \textbf{P 5\%} & \textbf{P 1\%} & \textbf{I 5\%} & \textbf{I 1\%} & \textbf{A 5\%} & \textbf{A 1\%} \\
\hline
Focal & 53.7 & 51.8 & 59.7 & 54.2 & 57.5 & 56.9 \\
Oversample & 59.7 & 50.0 & 59.0 & 52.1 & 58.8 & 57.9 \\
Undersample & 59.8 & 58.7 & 55.0 & 52.9 & 57.7 & 51.0 \\
\hline
\end{tabular}
\vspace{-5pt} 
\caption{P = plants, I = insects, A = animals}
\vspace{-9pt} 
\label{table:supervised_baselines}
\end{table}

We also evaluated several non-contrastive methods for mitigating data imbalance: majority-class undersampling, minority-class oversampling, and focal loss~\cite{lin2018focallossdenseobject}. As shown in \cref{table:supervised_baselines}, these methods consistently underperform compared to both weighted cross-entropy and our proposed SupCon-based solutions. We conjecture that the extreme imbalance in our scenarios contributes to these results. At a 1\% minority class ratio, undersampling yields only 188–249 training samples (spread across 80–125 species per dataset), lacking variability, while oversampling repeats the minority class up to 99\% of the time, leading to overfitting.

\section{Representation collapse during imbalanced binary supervised contrastive learning}
\label{app:imbalance_vs_performance}


~\cref{appfig:imbalance_vs_performance} illustrates the balanced test performance averaged over the three datasets - animals, insects, and plants - for both SupCon and weighted cross-entropy (CE) across varying levels of dataset imbalance. We observe that SupCon consistently outperforms CE when the imbalance is low, indicating its superiority in balanced or slightly imbalanced scenarios. As the imbalance increases, a transition point emerges between 10\% and 7.5\% imbalance percentages, where the performance of SupCon is equal to that of CE. Beyond this point, SupCon's balanced test accuracy declines more sharply than CE's. In extreme imbalance conditions (e.g., 5\%, 2.5\%, and 1\%) CE outperforms SupCon. These findings suggest that SupCon is highly effective in moderate imbalance conditions but struggles with extreme imbalance which is common to real world medical data.
\begin{figure}
    \centering
    \includegraphics[width=\linewidth]{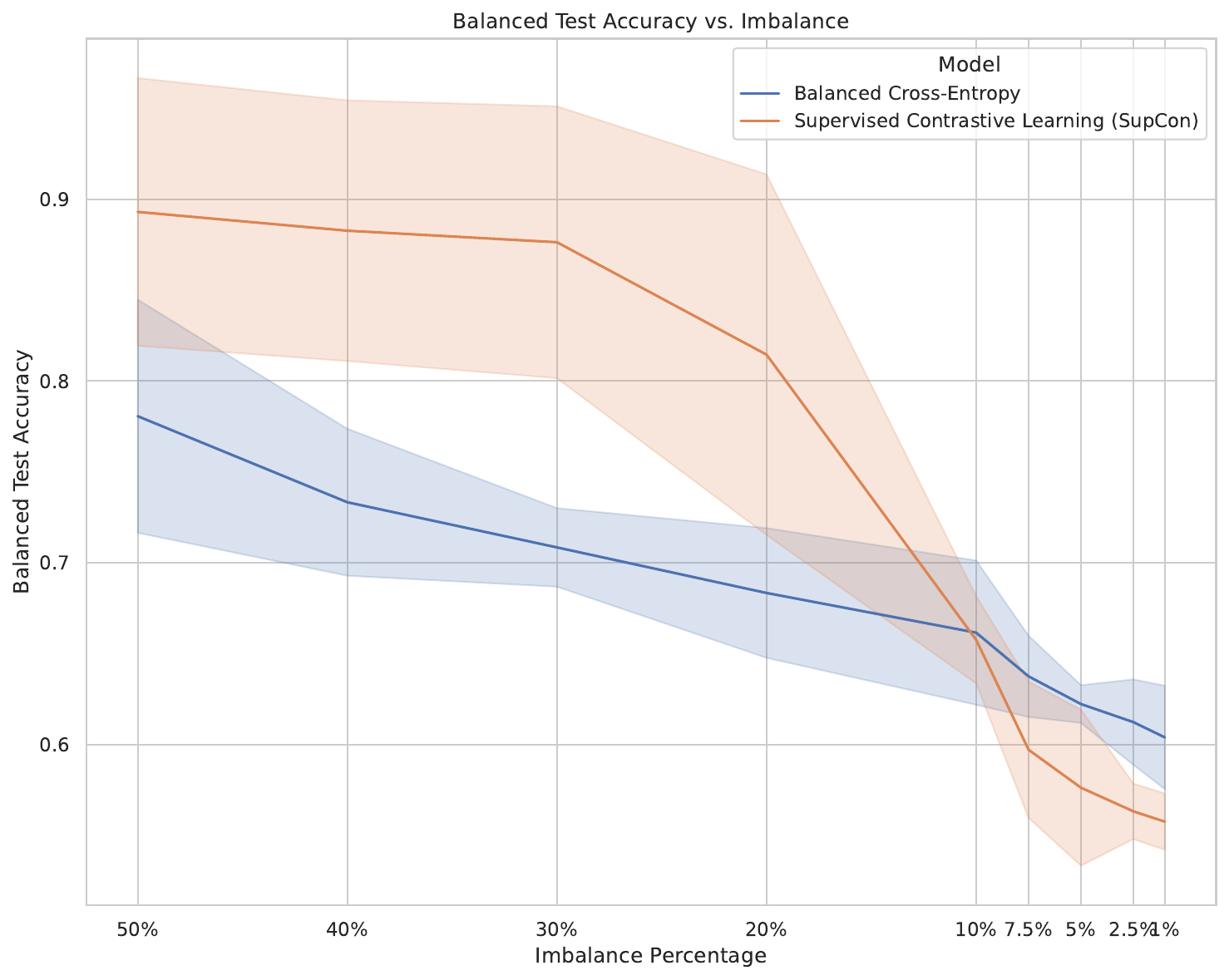}
    \caption{Balanced test performance averaged across three datasets - plants, insects, and animals - comparing SupCon and weighted cross-entropy under varying levels of dataset imbalance.}
    \label{appfig:imbalance_vs_performance}
\end{figure}

\FloatBarrier
\clearpage
\onecolumn
\section{Representation space analysis on insects and animals datasets}
\label{app:align_unif}

Similarly to the results in the main paper (see ~\cref{fig:sup_metrics}) SupCon exhibits a representation space collapses at high data imbalances on the insects and animals datasets. Despite the canonical SAD and CAD metrics being low, SAA and CAC correctly identify the collapse. We also see an indication in the elevated SAA and CAC values that the collapse of the insects dataset at 5\% imbalance was not quite as extreme as for the plants and animals datasets (62.6\% accuracy vs 56.2\% and 54.4\%). This trend is also visible but much less pronounced in SAD and CAD.

\begin{figure*}[htbp]
    \centering
    \includegraphics[width=.6\textwidth]{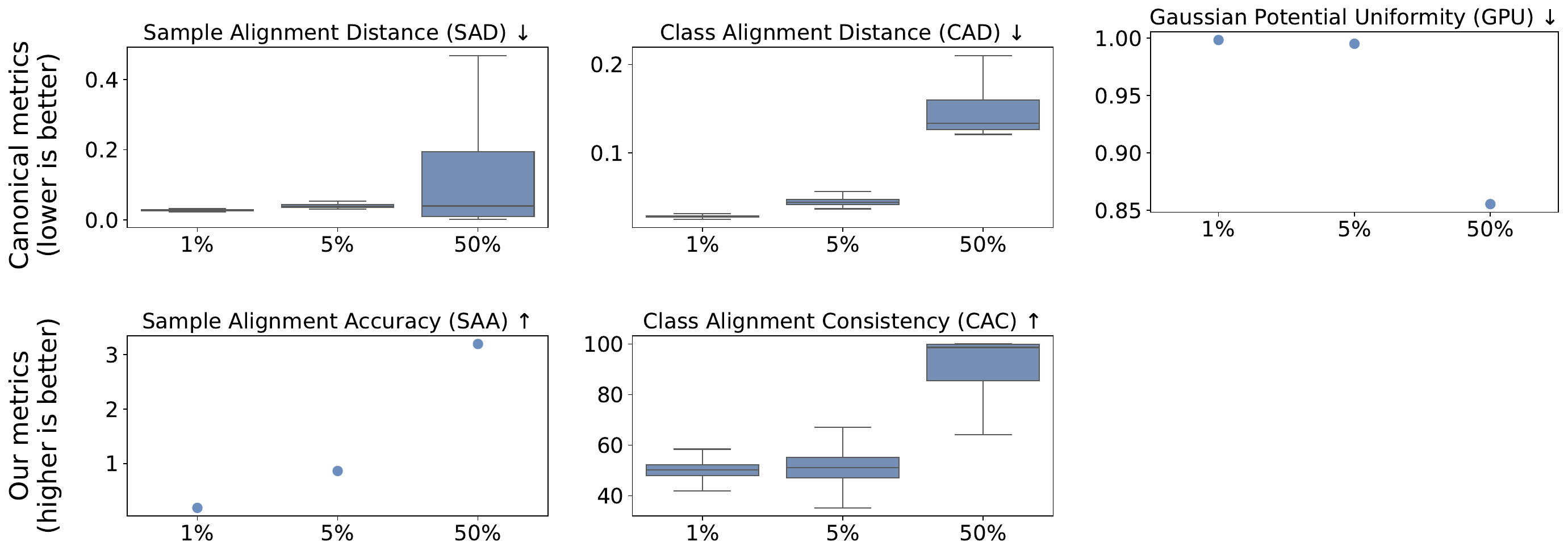}
    \caption{Analysis of SupCon's representation space learned from the insects dataset. }
    \label{fig:sup_metrics_animals}
\end{figure*}

\begin{figure*}[htbp]
    \centering
    \includegraphics[width=.6\textwidth]{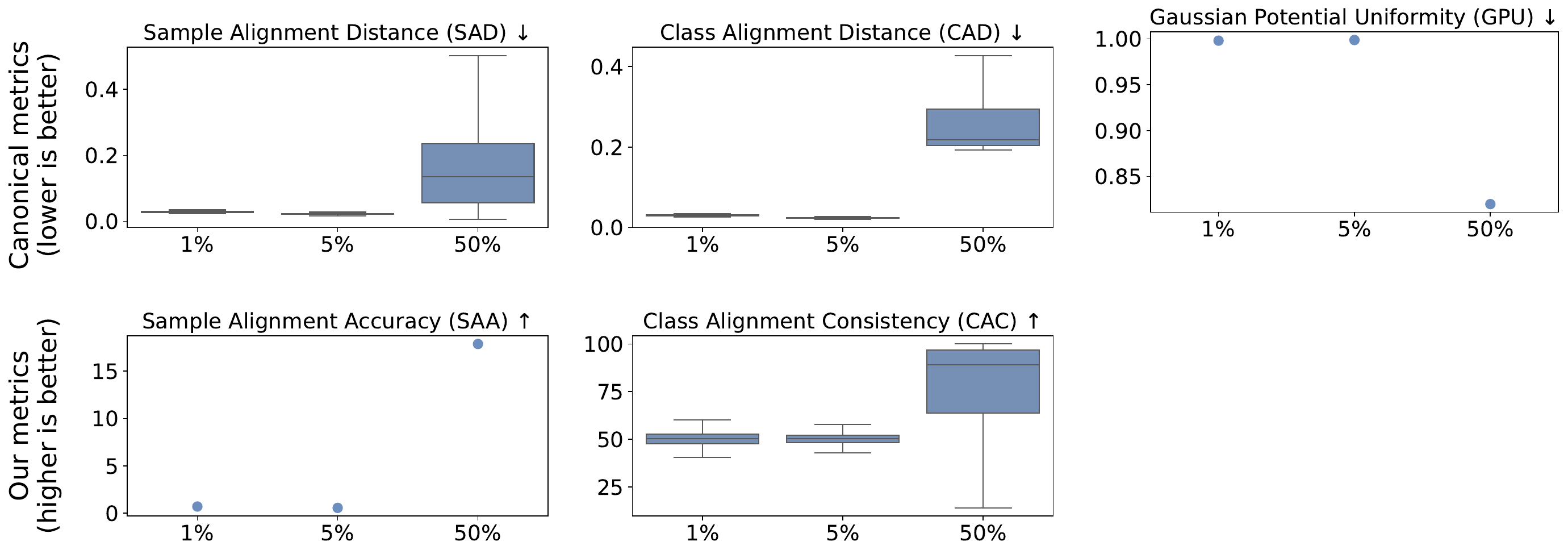}
    \caption{Analysis of SupCon's representation space learned from the animals dataset.}
    \label{fig:sup_metrics_insects}
\end{figure*}

\FloatBarrier
\clearpage
\onecolumn
\section{Proving supervised contrastive representation collapse for binary class imbalances}
\label{app:proof}

In the following, we demonstrate that in SupCon (see Equation \ref{eq:sup_con_loss}), the gradient of the output dimension can be effectively limited (upper bounded) by the count of positives associated with that sample. Consequently, an increase in the number of positives correlates with a decrease in the gradient magnitude.

\paragraph{Initial behavior of randomly initialized Resnet50}
\begin{table}[ht]
    \centering
    \small
    
    \begin{tabular}{c|c|c|c}
 
    \textbf{Model} & \textbf{Dataset} & \textbf{Mean Cosine Similarity} & \textbf{Standard Deviation} \\ \hline
    Randomly Initialized ResNet50 & Natural Images (Seed 1) & 0.9983 & 0.0012 \\ \hline
    Randomly Initialized ResNet50 & Natural Images (Seed 2) & 0.9986 & 0.0010 \\ \hline
    Randomly Initialized ResNet50 & Random Images & 0.9979 & 0.0010 \\ \hline
    Randomly Initialized ResNet50 & Pattern and Color Images & 0.9989 & 0.0027 \\ \hline
    Pretrained IMAGENET1K\_V2 & Natural Images & 0.0856 & 0.0635 
    \end{tabular}
    \caption{Initial Embedding Similarity in Randomly Initialized ResNet50}
    \label{tab:initial_embedding_similarity}
\end{table}
For our analysis, we make  assumptions regarding the initial state of the encoder output. At the start of training, the ResNet50 base encoder model \cite{DBLP:journals/corr/HeZRS15} is initialized with random weights. Liange et al. \cite{liang2022mindgapunderstandingmodality} have empirically shown that uninitialized ResNet models tend to map their inputs to almost identical vectors, with a cosine similarity exceeding 0.99. We confirm their findings in our own empirical study (\cref{tab:initial_embedding_similarity}).

In our study, we observed that regardless of the input type—be it natural images, random images, or artificially dissimilar images (such as inverted patterns and colors)—the randomly initialized model consistently mapped these diverse inputs to remarkably similar output embeddings. Based on these observations, we propose the following lemma:

\begin{lemma} \label{lemma:small_dis}
    Let $z_i,z_k \in \mathcal{S}^{128}$ be two projections of an uninitialized ResNet50 model and an uninitialized Projection layer. Then, 
    For a small $\varepsilon \in \mathbb{R}$, $\| z_i - z_k \| \leq \varepsilon$
    
\end{lemma}

Let \(\bm{w}_i\) denote the projection network output before normalization, i.e., \(\bm{z}_i = \frac{\bm{w}_i}{\|\bm{w}_i\|}\) \cite[p.~15]{DBLP:journals/corr/abs-2004-11362}. In our analysis, we focus on \(w_i\) rather than \(p_i\) since when \(w_i\) is small, even a minor modification followed by normalization results in a proportionally larger change. Consequently, the gradient magnitude for smaller values of \(w_i\) is amplified.

\noindent $A(i) \equiv I \setminus \{i\}$ is the set of all indices without the anchor $i$ in the multi-viewed batch. $P(i) \equiv \{ p \in A(i) : \tilde{y}_p = \tilde{y}_i \}$ is the set of indices of all positives in the multi-viewed batch distinct from $i$. $N(i) \equiv A(i) \setminus P(i)$ is the set of indices of all negatives in the multi-viewed batch.

Following Khosla \textit{et al.} \cite[p.~16]{DBLP:journals/corr/abs-2004-11362}, the gradient of the supervised loss in relation to $\bm{w}_i$, and restricted to $P(i)$ or $N(i)$ is
\begin{equation} \label{eq:2}
\frac{\partial \mathcal{L}^{sup}_i}{\partial \bm{w}_i}\Bigg|_{P(i)} = \frac{1}{\tau ||\bm{w}_i||} \sum_{p \in P(i)} (\bm{z}_p - (\bm{z}_i \cdot \bm{z}_p)\bm{z}_i)(P_{ip} - X_{ip})
\end{equation}

\begin{equation} \label{eq:3}
\frac{\partial \mathcal{L}^{sup}_i}{\partial \bm{w}_i}\Bigg|_{N(i)} = \frac{1}{\tau ||\bm{w}_i||} \sum_{n \in N(i)} (\bm{z}_n - (\bm{z}_i \cdot \bm{z}_n)\bm{z}_i)P_{in}
\end{equation}

\noindent
Summing Eqs. \ref{eq:2} and \ref{eq:3} gives us the gradient of the supervised loss with respect to $\bm{w}_i$:
\begin{equation}
\frac{\partial \mathcal{L}^{sup}_i}{\partial \bm{w}_i} = \frac{\partial \mathcal{L}^{sup}_i}{\partial \bm{w}_i}\Bigg|_{P(i)} + \frac{\partial \mathcal{L}^{sup}_i}{\partial \bm{w}_i}\Bigg|_{N(i)}
\end{equation}

\noindent
We further define following Khosla \textit{et al.} \cite[p.~16]{DBLP:journals/corr/abs-2004-11362}:

\begin{equation}
    P_{ix} \equiv \frac{\exp(\bm{z_i} \cdot \bm{z_x} / \tau)}{\sum_{a \in A(i)} \exp (\bm{z_i} \cdot \bm{z_a} / \tau)}
\end{equation}

\noindent
and

\begin{equation}
    X_{ip} \equiv \frac{1}{| P(i)|}
\end{equation}

\noindent

\begin{theorem} \label{theo:1}
Let us consider the context of Lemma \ref{lemma:small_dis}, where we assume for a small $\varepsilon \in \mathbb{R}$, that $\| \bm{z}_i - \bm{z}_j \| \leq \varepsilon$. Furthermore, given that $\bm{z}_i, \bm{z}_j \in \mathcal{S}^{128}$, we have $\| \bm{z}_i \| \cdot \| \bm{z}_j \| = 1$.
Under these conditions, the following inequality holds for the size of the output gradients:
    \begin{align*}
        \left\| \frac{\partial \mathcal{L}_i}{\partial \bm{w}_i} \right\|
        &\leq \frac{1}{\tau \|\bm{w}_i\|} (\varepsilon + \frac{1}{2} \varepsilon^2) \left( (1-\frac{|P(i)|}{|A(i)|}) \exp(-\varepsilon^2/ \tau) + (\exp(\varepsilon^2/ \tau) - 1) + (1 - \frac{|P(i)|}{|A(i)|})^2  \exp(\varepsilon^2/ \tau) \right) 
    \end{align*}
\end{theorem}

\begin{proof}

\[
\left\| \frac{\partial \mathcal{L}^{sup}_i}{\partial \bm{w}_i}\Bigg|_{P(i)} \right\| \leq \frac{1}{\tau \|\bm{w}_i\|} \sum_{p \in P(i)} \| \bm{z}_p - (\bm{z}_i \cdot \bm{z}_p)\bm{z}_i\| | P_{ip} - X_{ip} |
\]

First we show:
\[
| P_{ip} - X_{ip} | \leq \exp(-\varepsilon^2/ \tau)\frac{|N(i)|}{|P(i)||A(i)|} + (\exp(\varepsilon^2/ \tau)-1) \frac{1}{|P(i)|}
\]

with 
\begin{align*}
    &\frac{\exp((1- \varepsilon^2 / 2) / \tau)}{\exp((1+ \varepsilon^2 / 2) / \tau)} \frac{1}{|A(i)|} \\
    &\leq \frac{\exp((1- \| z_i - z_p \|^2 / 2) / \tau)}{\exp((1+ \varepsilon^2 / 2) / \tau)} \frac{1}{|A(i)|} && \text{\cref{lemma:small_dis} }\\
    &= \frac{\exp((1- (z_i^2-2z_iz_p+z_p^2) / 2) / \tau)}{\exp((1+ \varepsilon^2 / 2) / \tau)} \frac{1}{|A(i)|} && \\
    &= \frac{\exp((1- (\frac{1}{2}-z_iz_p+\frac{1}{2})) / \tau)}{\exp((1+ \varepsilon^2 / 2) / \tau)} \frac{1}{|A(i)|} && \\
    &= \frac{\exp((1- (1-z_iz_p)) / \tau)}{\exp((1+ \varepsilon^2 / 2) / \tau)} \frac{1}{|A(i)|} && \\
    &= \frac{\exp((\bm{z}_i \cdot \bm{z}_p) / \tau)}{\exp((1+ \varepsilon^2 / 2) / \tau)} \frac{1}{|A(i)|} && \\
    &\leq \frac{\exp(\bm{z}_i \cdot \bm{z}_p / \tau)}{\sum_{a \in A(i)} \exp(\bm{z}_i \cdot \bm{z}_a / \tau)}  \\
    &\leq \frac{\exp((1+ \varepsilon^2 / 2) / \tau)}{\exp((1- \varepsilon^2 / 2) / \tau)} \frac{1}{|A(i)|}
\end{align*}

if $P_{ip}\geq X_{ip}$:
\begin{align*}
        | P_{ip} - X_{ip} |
        &= P_{ip} - X_{ip} \\ 
        &= \frac{\exp(\bm{z}_i \cdot \bm{z}_p / \tau)}{\sum_{a \in A(i)} \exp(\bm{z}_i \cdot \bm{z}_a / \tau)} - \frac{1}{|P(i)|} \\
        &\leq \frac{\exp((1+ \varepsilon^2 / 2) / \tau)}{\exp((1- \varepsilon^2 / 2) / \tau)} \frac{1}{|A(i)|} - \frac{1}{|P(i)|} \\
        &= \frac{\exp((1+ \varepsilon^2 / 2) / \tau)}{\exp((1- \varepsilon^2 / 2) / \tau)} \frac{1}{|P(i)|+|N(i)| } - \frac{1}{|P(i)|} \\
        &\leq \frac{\exp((1+ \varepsilon^2 / 2) / \tau)}{\exp((1- \varepsilon^2 / 2) / \tau)} \frac{1}{|P(i)|} - \frac{1}{|P(i)|} \\
        &= \exp(\varepsilon^2 / \tau) \frac{1}{|P(i)|} - \frac{1}{|P(i)|} \\
        &= (\exp(\varepsilon^2 / \tau) -1)\frac{1}{|P(i)|} 
\end{align*}
if $X_{ip}\geq P_{ip}$:
\begin{align*}
        | P_{ip} - X_{ip} |
        &= X_{ip} - P_{ip} \\
        &\leq \frac{1}{|P(i)|}- \frac{\exp((1- \varepsilon^2 / 2) / \tau)}{\exp((1+ \varepsilon^2 / 2) / \tau)} \frac{1}{|A(i)|} \\
        &= \frac{1}{|P(i)|}+\exp(-\varepsilon^2 / \tau) \frac{-1}{|A(i)|} \\
        &= \frac{1}{|P(i)|}+\exp(-\varepsilon^2 / \tau) \left[ \frac{-|P(i)|}{|P(i)||A(i)|} + \frac{|A(i)|}{|P(i)||A(i)|} - \frac{|A(i)|}{|P(i)||A(i)|} \right ]\\
        &= \frac{1}{|P(i)|}+\exp(-\varepsilon^2 / \tau) \left[\frac{|A(i)|-|P(i)|}{|P(i)||A(i)|} - \frac{|A(i)|}{|P(i)||A(i)|} \right] \\
        &= \frac{1}{|P(i)|}+\exp(-\varepsilon^2 / \tau) \left[ \frac{|A(i)|-|P(i)|}{|P(i)||A(i)|} - \frac{1}{|P(i)|} \right]\\
        &=\exp(-\varepsilon^2/ \tau) \frac{|N(i)|}{|A(i)||P(i)|} + (1-\exp(-\varepsilon^2 / \tau)) \frac{1}{|P(i)|} \\
        &\leq \exp(-\varepsilon^2/ \tau) (\frac{|N(i)|}{|A(i)||P(i)|} + (\exp(\varepsilon^2 / \tau) -1)\frac{1}{|P(i)|}) 
    \end{align*}

    \[ \Longrightarrow
| P_{ip} - X_{ip} | \leq \exp(-\varepsilon^2/ \tau)(\frac{|N(i)|}{|P(i)||A(i)|} + (\exp(\varepsilon^2/ \tau)-1) \frac{1}{|P(i)|})
\]

\begin{align*}
    \| \bm{z}_p - (\bm{z}_i \cdot \bm{z}_p) \bm{z}_i \| 
    &= \| \bm{z}_p - \bm{z}_i + \bm{z}_i - (\bm{z}_i \cdot \bm{z}_p) \bm{z}_i \| \\
    & \leq \| \bm{z}_p - \bm{z}_i \| + | 1- (\bm{z}_i \cdot \bm{z}_p) | \|\bm{z}_i \| \\
    & \leq \varepsilon + \left| 1 - \mathbf{z}_i \cdot \mathbf{z}_p \right| \\
    &= \varepsilon + \left| 1 + \frac{1}{2}(\mathbf{z}_i - \mathbf{z}_p) \cdot (\mathbf{z}_i - \mathbf{z}_p) - \frac{1}{2} \mathbf{z}_i \cdot \mathbf{z}_i - \frac{1}{2} \mathbf{z}_p \cdot \mathbf{z}_p \right| \\
    &\leq \varepsilon +  \left| \frac{1}{2} (\mathbf{z}_i - \mathbf{z}_p) \cdot (\mathbf{z}_i - \mathbf{z}_p) \right| \\
    & \leq \varepsilon +  \frac{1}{2}  \varepsilon^2
\end{align*}

\begin{align*}
\left\| \frac{\partial \mathcal{L}^{sup}_i}{\partial \bm{w}_i}\Bigg|_{N(i)} \right\| &\leq \frac{1}{\tau \|\bm{w}_i\|} \sum_{n \in N(i)} \| \bm{z}_n - (\bm{z}_i \cdot \bm{z}_n)\bm{z}_i\| | P_{in} | \\
&\leq \frac{1}{\tau \|\bm{w}_i\|} \frac{|N(i)|}{|A(i)|}(\varepsilon+\frac{1}{2}\varepsilon^2)\exp(\varepsilon^2 / \tau)
\end{align*}

Finally,
\begin{align*}
\left\| \frac{\partial \mathcal{L}^{sup}_i}{\partial \bm{w}_i} \right\| &\leq \left\| \frac{\partial \mathcal{L}^{sup}_i}{\partial \bm{w}_i}\Bigg|_{P(i)} \right\|+ \left\| \frac{\partial \mathcal{L}^{sup}}{\partial \bm{w}_i}\Bigg|_{N(i)}\right\| \\
&\leq \frac{1}{\tau \|\bm{w}_i\|} \left( \sum_{p \in P(i)} \| \bm{z}_p - (\bm{z}_i \cdot \bm{z}_p) \bm{z}_i \| | P_{ip} - X_{ip} | + \sum_{n \in N(i)} \| \bm{z}_n - (\bm{z}_i \cdot \bm{z}_n) \bm{z}_i \| | P_{in} | \right) \\
&\leq \frac{1}{\tau \|\bm{w}_i\|} \left( \sum_{p \in P(i)} (\varepsilon + \frac{1}{2} \varepsilon^2) \left( \exp(-\varepsilon^2/ \tau) (\frac{|N(i)|}{|A(i)||P(i)|} + (\exp(\varepsilon^2 / \tau) -1)\frac{1}{|P(i)|}) \right) \right. \\
&\phantom{=} \left. + (\varepsilon + \frac{1}{2} \varepsilon^2) \frac{|N(i)|}{|A(i)|} \exp(\varepsilon^2/ \tau) \right) \\
&\leq \frac{1}{\tau \|\bm{w}_i\|} \left( |P(i)| (\varepsilon + \frac{1}{2} \varepsilon^2) \left( \exp(-\varepsilon^2/ \tau) (\frac{|N(i)|}{|P(i)||A(i)|} + (\exp(\varepsilon^2/ \tau) -1) \frac{1}{|P(i)|} )\right) \right. \\
&\phantom{=} \left. +  (\varepsilon + \frac{1}{2} \varepsilon^2) \frac{|N(i)|}{|A(i)|} \exp(\varepsilon^2/ \tau) \right) \\
&= \frac{1}{\tau \|\bm{w}_i\|} (\varepsilon + \frac{1}{2} \varepsilon^2) \left( \exp(-\varepsilon^2/ \tau) (\frac{|N(i)| }{|A(i)|} + (\exp(\varepsilon^2/ \tau) - 1) \frac{|P(i)|}{|P(i)|}) \right) \\
&\phantom{=} + \frac{1}{\tau \|\bm{w}_i\|} (\varepsilon + \frac{1}{2} \varepsilon^2) \left( \frac{|N(i)|}{|A(i)|} \exp(\varepsilon^2/ \tau) \right) \\
&= \frac{1}{\tau \|\bm{w}_i\|} (\varepsilon + \frac{1}{2} \varepsilon^2) \left( \frac{|N(i)|}{|A(i)|} \exp(-\varepsilon^2/ \tau) + \exp(-\varepsilon^2/ \tau)(\exp(\varepsilon^2/ \tau) - 1) + \frac{|N(i)|}{|A(i)|} \exp(\varepsilon^2/ \tau) \right) \\
&= \frac{1}{\tau \|\bm{w}_i\|} (\varepsilon + \frac{1}{2} \varepsilon^2) \left( (1-\frac{|P(i)|}{|A(i)|}) \exp(-\varepsilon^2/ \tau) + \exp(-\varepsilon^2/ \tau)(\exp(\varepsilon^2/ \tau) - 1) + (1 - \frac{|P(i)|}{|A(i)|}) \exp(\varepsilon^2/ \tau) \right) 
\end{align*}

\end{proof}

Based on the theorem's conclusions, we see that the gradients for the loss function in supervised contrastive learning are upper-bounded by the number of positive samples. In the context of severe class imbalances, the increment in the number of positive samples (from the majority class) may dominate the output vector. This dominance can constrain the gradient magnitudes, causing them to become too small to induce effective weight updates in the network.

\clearpage
\twocolumn

\section{Ablations}
\label{app:ablations}

\subsection{Ablations on temperature and batch size}
\label{app:temp_bs_ablations}
Experiments on temperature reveal that both fixes are robust across a range of temperature settings, with optimal results observed for temperatures between 0.1 and 0.5. 
While low to medium temperatures do not alleviate the collapse in SupCon, very high temperatures can mitigate collapse issues in moderately imbalanced scenarios; however, this comes at a cost, resulting in an accuracy that is 19\% lower than our proposed method. \cref{fig:ablation_tem_batchsize}

Furthermore, unlike supervised contrastive learning in balanced multi-class datasets, we find that increasing batch sizes negatively affects performance. We attribute this degradation to the larger number of positive pairs per sample introduced by bigger batches, leading to collapse. A detailed theoretical justification of this phenomenon is provided in \cref{app:proof}.

\begin{figure}[h]
    \centering
    \includegraphics[width=\linewidth]{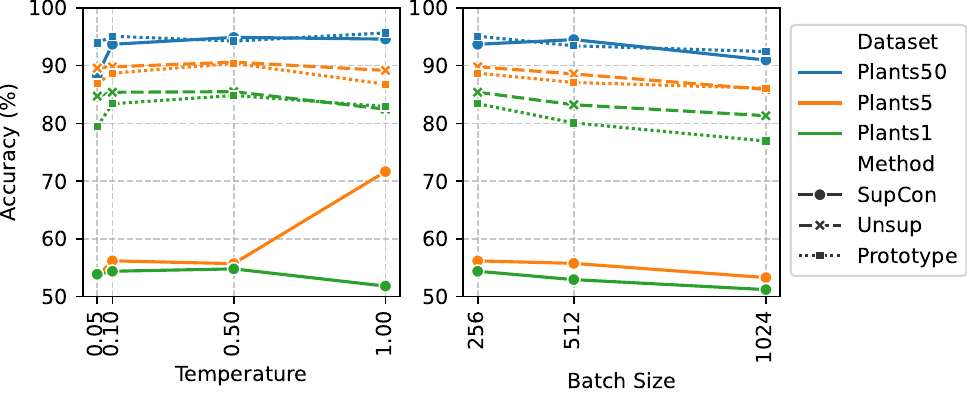}
    \caption{Ablation of our fixes and SupCon for different batch sizes and temperatures using the plants dataset. Our fixes are robust across a range of temperatures and batch sizes, though increasing batch size typically degrades performance slightly.}
    \label{fig:ablation_tem_batchsize}
\end{figure}

\subsection{Supervised minority ablation}
\label{app:sup_min_ablation}

In \cref{tab:ablation}, we investigate how varying levels of supervision in the majority class impact performance on an imbalanced dataset, while keeping the supervision level fixed in the minority class (see \cref{app:def:minsup}). The study was conducted using the insects dataset composed of 5\% minority samples and 95\% majority samples. This shows that our strategy of full supervision in the minority and no supervision in the majority performs best in these strong imbalance scenarios. A notable drop in performance occurs between 5\% and 1\% supervision where the representations collapse. This ablation study is similar to KCL \cite{kang2021exploring} under varying levels of K. The results are consistent with our KCL baselines as we find that a larger K in a batch is harmful for downstream utility. For a batch size of 256, 5\% supervision already translates to $K=12.8$. 
 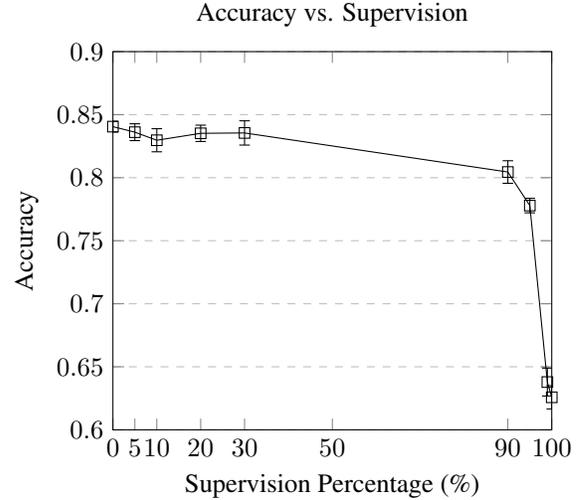
\begin{figure}[H]
\centering

\begin{tikzpicture}
\begin{axis}[
    title={Accuracy vs. Supervision},
    xlabel={Supervision Percentage (\%)},
    ylabel={Accuracy},
    xmin=0, xmax=100,
    ymin=0.6, ymax=0.9,
    xtick={0,5,10,20,30,50,90,100},
    ytick={0.6,0.65,0.7,0.75,0.8,0.85,0.9},
    legend pos=south east,
    ymajorgrids=true,
    grid style=dashed,
    width=0.7\linewidth,
    scale only axis,
]

\addplot[
    color=black,
    mark=square,
    error bars/.cd,
    y dir=both, y explicit,
]
coordinates {
    (0,0.8405) +- (0,0.004304)
    (5,0.8361) +- (0,0.006683)
    (10,0.8297) +- (0,0.009154)
    (20,0.8352) +- (0,0.00648)
    (30,0.8355) +- (0,0.009624)
    (90,0.8045) +- (0,0.008953)
    (95,0.7779) +- (0,0.005765)
    (99,0.6379) +- (0,0.01106)
    (100,0.6258) +- (0,0.0092)
};

\end{axis}
\end{tikzpicture}
\caption{Increasing supervision in the majority class vs. balanced test accuracy.}
\end{figure}
\begin{table}[ht]
\centering
\footnotesize
\begin{tabular}{cc}

\toprule    
Supervision $\theta$ & Accuracy (\%) \\
\midrule
    0\% &  $\textbf{84.05} \pm 0.43$ \\
    5\% &  $83.61 \pm 0.66$ \\
    10\%  & $82.97 \pm 0.91$ \\
    20\%  &  $83.52  \pm 0.64$ \\
    30\%  &  $83.55 \pm 0.96$\\
    90\%  &  $80.45 \pm 0.89$\\
    95\%  &  $77.79 \pm 0.57$ \\
    99\%  &  $63.79 \pm 1.10$ \\
    100\% & $62.58 \pm 0.92$ \\
\bottomrule
\end{tabular}

\caption[Ablation Study]{For our Supervised Minority fix, we show the effects of increasing supervision in the majority class on the insects dataset with 5\% imbalance. No amount of supervision in the majority class improves downstream performance.}
\label{tab:ablation}
\end{table}

\subsection{Supervised majority ablation}
\label{app:sup_maj_ablation}
\begin{table}[h]
\centering
\begin{tabular}{|c|c|c|}
\hline
\multirow{2}{*}{\textbf{Majority Supervision (\%)}} & \multicolumn{2}{c|}{\textbf{Label Imbalance}} \\
\cline{2-3}
 & \textbf{5\%} & \textbf{1\%} \\
\hline
10 & 67.44 & 55.93 \\ 
50 & 54.72 & 53.11 \\ 
100 & 52.16 & 51.58 \\ 
\hline
\end{tabular}
\caption{Test accuracy when supervision is applied to the majority class instead of the minority class, evaluated at different supervision levels on the Plants dataset with 5\% and 1\% imbalance.}
\label{tab:performance_sup_maj_ablation}
\end{table}

We evaluated using SupCon loss for the majority class and NT-Xent loss for the minority class. With no supervision in the minority class, even mild supervision in the majority class failed to train effectively. 

\section{UMAP Visualization} \label{app:umap_supervised}

UMAP~\cite{mcinnes2020umapuniformmanifoldapproximation} is a dimensionality reduction technique that is widely used for visualizing high-dimensional data. We employ UMAP to visualize the embeddings of all three datasets - \textit{plants}, \textit{insects}, and \textit{animals} - under varying levels of class imbalance, using unseen test data (see \cref{fig:umap_plants,fig:umap_bee,fig:umap_animals}).

The UMAP visualizations of the SupCon embedding spaces corroborate our findings from the main paper: the embedding space collapses under strong class imbalances, resulting in diminished utility. Even in the balanced case, the two classes are not distinctly separated. It is important to note that UMAP represents pairwise distances in a relative manner, which can obscure the visualization of an embedding space collapsing to a single vector. The relative scaling in UMAP means that even minimal differences between embeddings can appear more pronounced, masking the extent of the collapse.

In contrast, the \textit{Supervised Prototype Fix} and the \textit{Supervised Minority Fix} methods exhibit clear class clusters and separation across all levels of imbalance. This observation aligns with our theoretical illustrations presented in Figure~\ref{fig:sup_fixes}.


\begin{figure*}[h]
  \centering
  \includegraphics[width=0.75\linewidth]{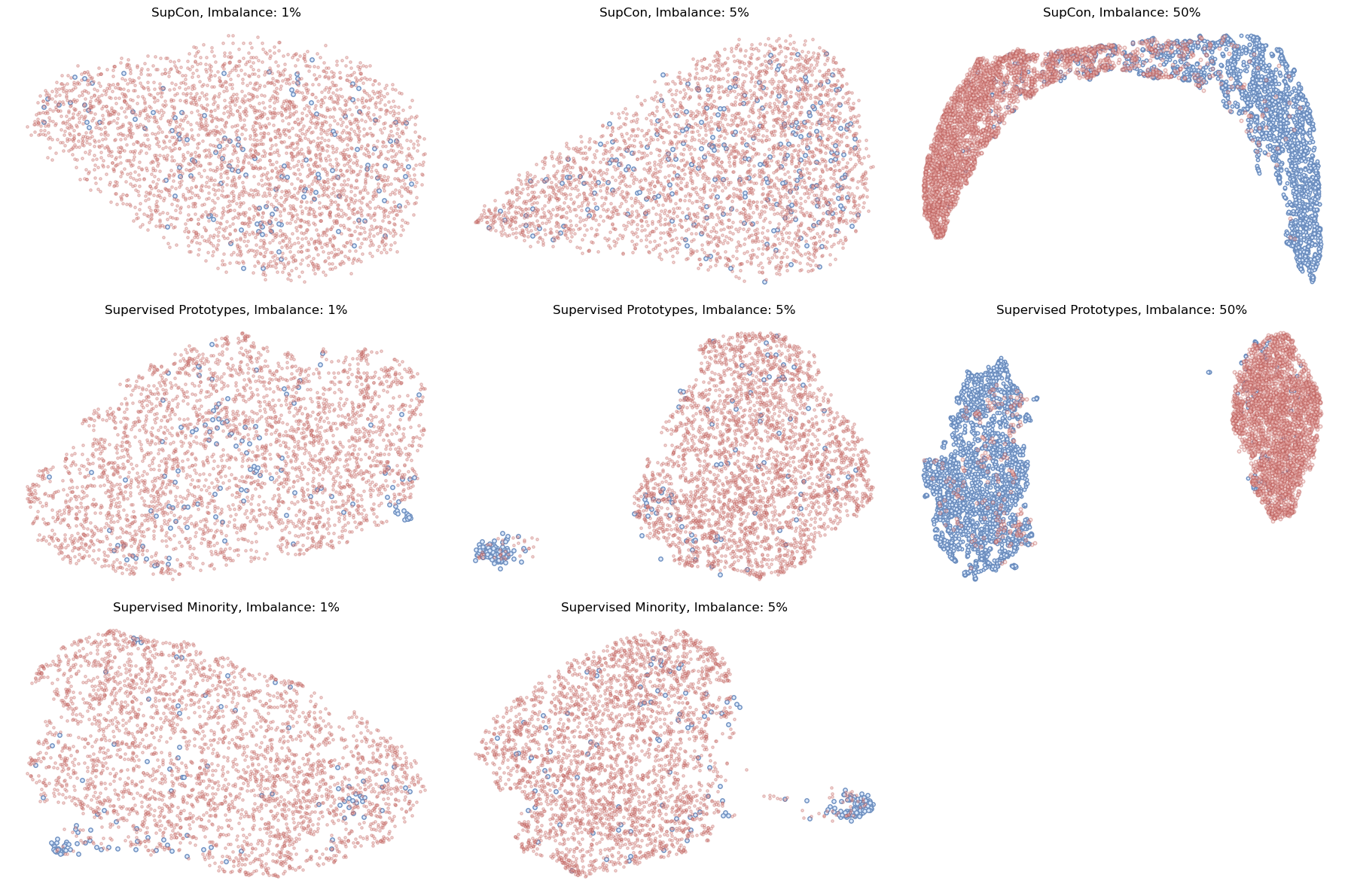}
  \caption{UMAP visualization of projection space after supervised pre-training on the plants dataset.}
  \label{fig:umap_plants}
\end{figure*}

\begin{figure*}[h]
  \centering
  \includegraphics[width=0.75\linewidth]{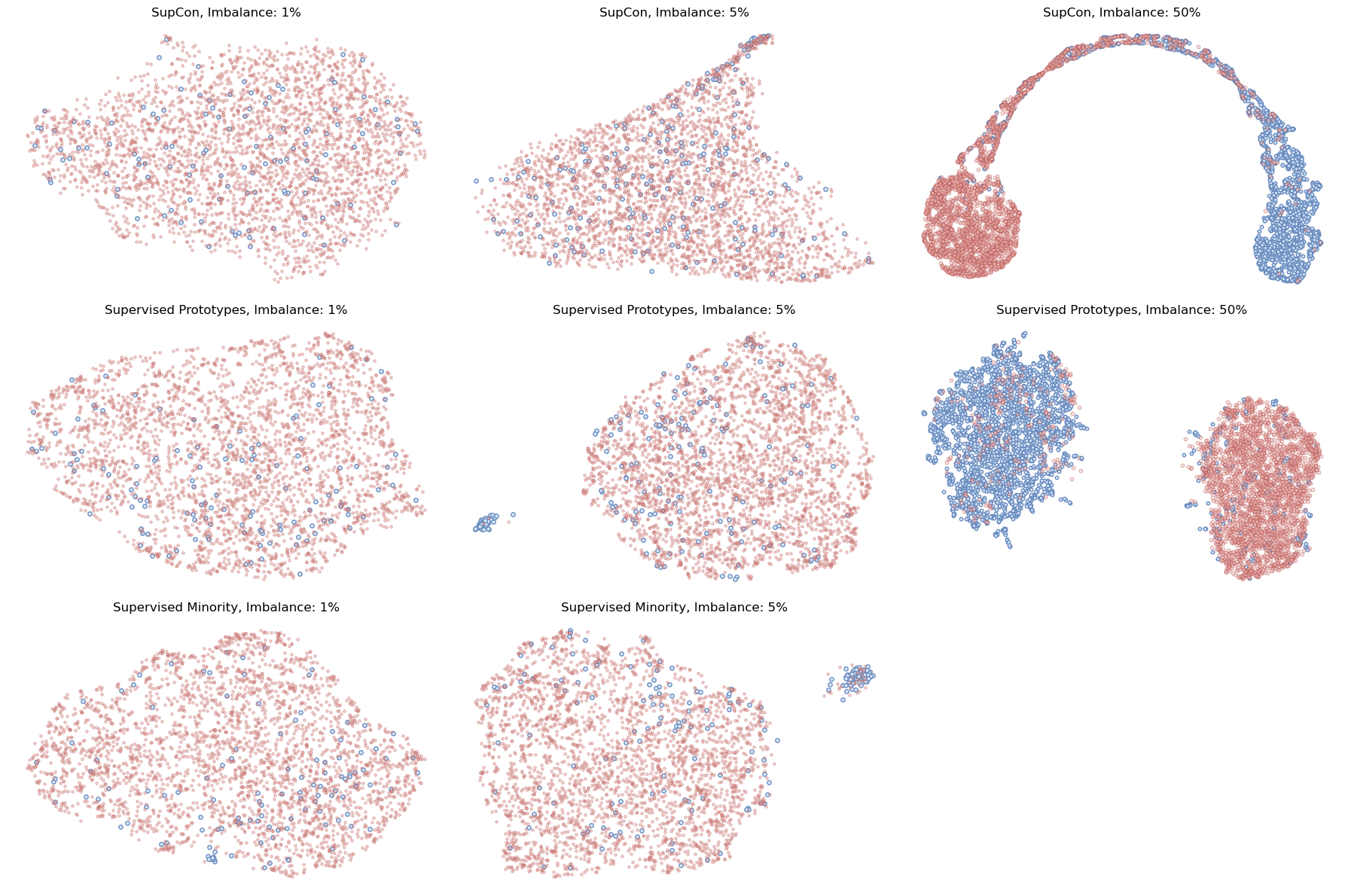}
  \caption{UMAP visualization of projection space after supervised pre-training on the insects dataset.}
  \label{fig:umap_bee}
\end{figure*}

\begin{figure*}[h]
  \centering
  \includegraphics[width=0.75\linewidth]{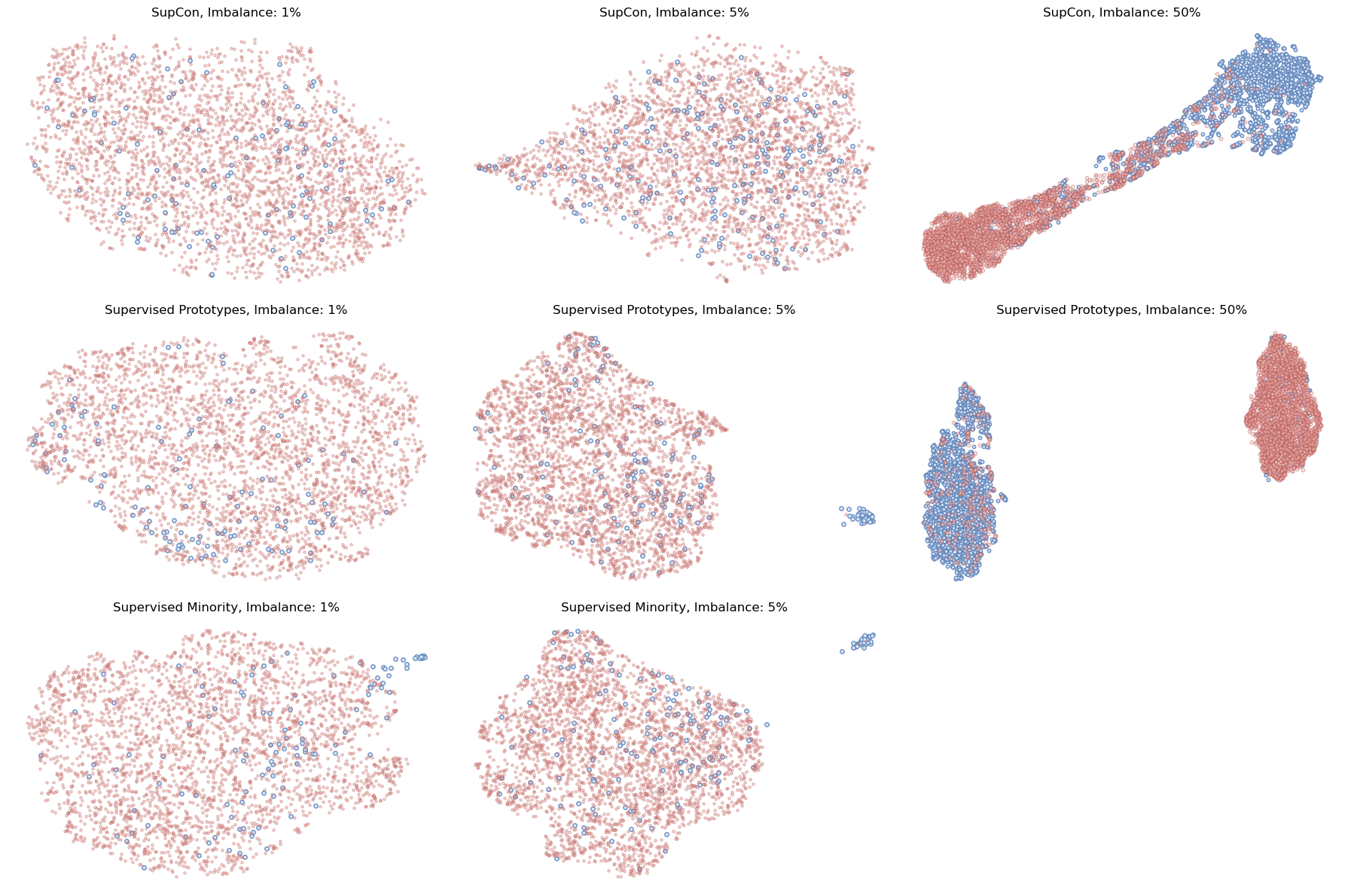}
  \caption{UMAP visualization of projection space after supervised pre-training on the animals dataset.}
  \label{fig:umap_animals}
\end{figure*}

\end{document}